\def\year{2021}\relax
\documentclass[letterpaper]{article} 
\usepackage{aaai21}  
\usepackage{times}  
\usepackage{helvet} 
\usepackage{courier}  
\usepackage[hyphens]{url}  
\usepackage{graphicx} 
\usepackage{natbib}  
\usepackage{caption} 

\usepackage{booktabs}

\usepackage{subfigure}

\usepackage{amsmath,amssymb,mathtools,amsfonts}
\usepackage{amsthm}
\usepackage{mathrsfs}

\usepackage{algorithm}
\usepackage{algorithmic}
\usepackage{color, soul}

\newcounter{ass_counter}
\newcounter{thm_counter}

\newcounter{lem_counter}
\newcounter{remark_counter}
\newtheorem{theorem}[thm_counter]{Theorem}
\newtheorem{lemma}[lem_counter]{Lemma}

\newtheorem{assumption}[ass_counter]{Assumption}
\newtheorem{remark}[remark_counter]{Remark}
\newtheorem{definition}{Definition}

\usepackage{bm}

\usepackage{makecell}
\usepackage{footnote}

\usepackage{wrapfig}
\usepackage{lipsum}

\title{STL-SGD: Speeding Up Local SGD with Stagewise Communication Period}
\author{
    Shuheng Shen\textsuperscript{\rm 1, 3}\footnote{Equal contribution.},
    Yifei Cheng\textsuperscript{\rm 2}\footnotemark[1],
    Jingchang Liu\textsuperscript{\rm 5},
    Linli Xu\textsuperscript{\rm 1, 4}\thanks{The corresponding author.}
    \\
}
\affiliations{
    \textsuperscript{\rm 1}Anhui Province Key Laboratory of Big Data Analysis and Application,\\
    School of Computer Science and Technology, University of Science and Technology of China\\
    \textsuperscript{\rm 2}School of Data Science, University of Science and Technology of China\\
    \textsuperscript{\rm 3} Ant Financial Services Group~~~~~~~~~ \textsuperscript{\rm 4} IFLYTEK Co., Ltd.\\
    \textsuperscript{\rm 5}Department of Computer Science and Engineering, Hong Kong University of Science and Technology\\
    shuheng.ssh@antgroup.com, chengyif@mail.ustc.edu.cn, jliude@cse.ust.hk, linlixu@ustc.edu.cn

}

\begin{document}

\maketitle

\begin{abstract}
    Distributed parallel stochastic gradient descent algorithms are workhorses for large scale machine learning tasks. 
    Among them, local stochastic gradient descent (Local SGD) has attracted significant attention due to its low communication complexity. 
    Previous studies prove that the communication complexity of Local SGD with a fixed or an adaptive communication period  is in the order of $O (N^{\frac{3}{2}} T^{\frac{1}{2}})$ and $O (N^{\frac{3}{4}} T^{\frac{3}{4}})$ when the data distributions on clients are identical (IID) or otherwise (Non-IID), where $N$ is the number of clients and $T$ is the number of iterations.
    In this paper, to accelerate the convergence by reducing the communication complexity,
    we propose \textit{ST}agewise \textit{L}ocal \textit{SGD} (STL-SGD), which increases the communication period gradually along with decreasing learning rate.
    We prove that STL-SGD can keep the same convergence rate and linear speedup as mini-batch SGD.
    In addition, as the benefit of increasing the communication period, when the objective is strongly convex or satisfies the Polyak-\L ojasiewicz condition, the communication complexity of STL-SGD is $O (N \log{T})$ and $O (N^{\frac{1}{2}} T^{\frac{1}{2}})$ for the IID case and the Non-IID case respectively, achieving significant improvements over Local SGD. 
    Experiments on both convex and non-convex problems demonstrate the superior performance of STL-SGD.
  \end{abstract}

  \section{Introduction}
  \label{introduction}
  We consider the task of distributed stochastic optimization, which employs $N$ clients to solve the following empirical risk minimization problem:
  \begin{equation} \label{basic_object}
  \min_{x \in R^d} f(x) := \frac{1}{N} \sum_{i=1}^N f_i(x),
  \end{equation}
  where $f_i(x) := \frac{1}{|\mathcal{D}_i|} \sum_{\xi \in \mathcal{D}_i} f(x, \xi)$ is the local objective of client $i$. 
  $\mathcal{D}_i$'s denote the data distributions among clients, which can be possibly different.
  Specifically, the scenario where $\mathcal{D}_i$'s are identical corresponds to a central problem of traditional distributed optimization. 
  When they are not identical, 
  (\ref{basic_object}) captures the federated learning setting~\cite{mcmahan2017communication,kairouz2019advances,lyu2020threats}, where the local data in each mobile client is independent and private, resulting in high variance of the data distributions.
  
  As representatives of distributed stochastic optimization methods, traditional Synchronous SGD (SyncSGD)~\cite{dekel2012optimal,ghadimi2013stochastic} and Asynchronous SGD (AsyncSGD)~\cite{agarwal2011distributed,lian2015asynchronous} achieve linear speedup theoretically with respect to the number of clients.  
  Nevertheless, for both SyncSGD and AsyncSGD, communication needs to be conducted at each iteration and $O (d)$ parameters are communicated each time, incurring significant communication cost which restricts the performance in terms of time speedup.
  To address this dilemma, distributed algorithms with low communication cost, either by decreasing the communication frequency~\cite{wang2018cooperative,stich2018local,yu2019parallel,shen2019faster} 
  or by reducing the communication bits in each round~\cite{alistarh2017qsgd,stich2018sparsified,tang2019doublesqueeze}, become widely applied for large scale training.
  
  Among them, Local SGD~\cite{stich2018local} (also called FedAvg~\cite{mcmahan2017communication}), which conducts communication every $k$ iterations, enjoys excellent theoretical and practical performance~\cite{lin2018don,stich2018local}.  
  In the IID case and the Non-IID case, the communication complexity of Local SGD is respectively proved to be $O(N^{\frac{3}{2}} T^{\frac{1}{2}})$~\cite{wang2018cooperative,stich2018local} and $O(N^{\frac{3}{4}} T^{\frac{3}{4}})$~\cite{yu2019parallel,shen2019faster}, while the linear speedup is maintained.
  When the objective satisfies the Polyak-\L ojasiewicz condition~\cite{karimi2016linear}, 
  \cite{haddadpour2019local} provides
  a tighter theoretical analysis which shows that the communication complexity of Local SGD is $O(N^{\frac{1}{3}} T^{\frac{1}{3}})$. 
  In terms of the communication period $k$, most previous studies of Local SGD 
  choose to fix it through the iterations. 
  In contrast, 
  \cite{wang2018adaptive} suggests
  using an 
  adaptively decreasing $k$ when the learning rate is fixed, and 
  \cite{haddadpour2019local} proposes
  an adaptively increasing $k$ as the iterations go on. 
  Nevertheless, none of them achieve a communication complexity lower than $O(N^{\frac{1}{3}} T^{\frac{1}{3}})$. 
  For strongly convex objectives, if a fixed learning rate is adopted, Local SGD with fixed communication period is proved to achieve $O(N\log{(NT)})$~\cite{stich2019errorfeedback,bayoumi2020tighter} communication complexity. However, the fixed learning rate results in suboptimal convergence rate $O(\frac{\log{T}}{NT})$.
  It remains an open problem as to whether the communication complexity can be further reduced with a varying $k$ 
  when the optimal convergence rate $O(\frac{1}{NT})$ is maintained,
  to which this paper provides an affirmative solution.

  \textbf{Main Contributions.} We propose Stagewise Local SGD (STL-SGD), which adopts a stagewisely increasing communication period 
  , and make the following contributions:
  \begin{itemize}
  \item We first prove that Local SGD achieves $O(\frac{1}{\sqrt{NT}})$ convergence when the objective is general convex. A novel insight is that, the convergence rate $O(\frac{1}{\sqrt{NT}})$ can be attained when setting $k$ to be $O(\frac{1}{\eta N})$ and $O(\frac{1}{\sqrt{\eta N}})$ in the IID case and the Non-IID case respectively, where $\eta$ is the learning rate. This indicates that the communication period is negatively relevant to the learning rate.
  \item Taking Local SGD as a subalgorithm and tuning its parameters stagewisely, we propose $\text{STL-SGD}^{sc}$ for 
  strongly
  convex problems, which geometrically increases the communication period along with decreasing learning rate. 
  We prove that $\text{STL-SGD}^{sc}$ achieves 
  $O(\frac{1}{NT})$ convergence rate
  with communication complexities $O(N \log{T})$ and $O (N^{\frac{1}{2}} T^{\frac{1}{2}})$ for the IID case and the Non-IID case, respectively.
  \item For non-convex problems, we propose the $\text{STL-SGD}^{nc}$ algorithm, which uses Local SGD to optimize a regularized objective $f_{x_s}^{\gamma}(\cdot)$ inexactly at each stage. When the Polyak-\L ojasiewicz condition holds, the same communication complexity as in 
  strongly
  convex problems is achieved. For general non-convex problems, we prove that $\text{STL-SGD}^{nc}$ achieves the linear speedup with communication complexities $O(N^{\frac{3}{2}} T^{\frac{1}{2}})$ and $O (N^{\frac{3}{4}} T^{\frac{3}{4}})$ for the IID case and the Non-IID case, respectively.
  \end{itemize} 
  
  \begin{table*}[!h]
    \caption{A comparison of the results in this paper and previous state-of-the-art results of Local SGD and its variants. 
    Regarding orders of convergence rate and communication complexity, we highlight the dependency on $T$ (the number of iterations), $N$ (the number of clients) and $k$ (communication period). Previous results may depend on some extra assumptions, which include: (1) an upper bound for gradient, (2) an upper bound for the gradient variance among clients and (3) an upper bound for the gradient diversity, which are shown in the last column.}
    \label{comparison}
    \vskip 0.05in
    \begin{center}
    {\small
    \begin{tabular}{l l l l l l}
    \toprule
    \makecell[c]{Algorithms} & \makecell[c]{Objectives} & \makecell[c]{Convergence \\Rate} & \makecell[c]{Communication \\Complexity} & \makecell[c]{Data \\ Distributions} & \makecell[c]{Extra \\Assumptions}\\ 
    \midrule
    \makecell[c]{Local SGD~\cite{stich2018local}} & 
    \makecell[c]{Strongly Convex} 
    & \makecell[c]{$O(\frac{1}{NT})$} & \makecell[c]{$O(N^{\frac{1}{2}} T^{\frac{1}{2}})$} & \makecell[c]{IID} & \makecell[c]{(1)} \\
    \makecell[c]{Local SGD~\cite{stich2019errorfeedback}}\footnotemark[1] & 
    \makecell[c]{Strongly Convex} 
    & \makecell[c]{$O(\frac{\log{T}}{NT})$} & \makecell[c]{$O(N\log{(NT)})$} & \makecell[c]{IID} & \makecell[c]{No} \\
    \makecell[c]{\textbf{STL-SGD}} & 
    \makecell[c]{\textbf{Strongly Convex}} 
    & \makecell[c]{\bm{$O(\frac{1}{NT})$}} & \makecell[c]{\bm{$O(N \log{T})$}} & \makecell[c]{\textbf{IID}} & \makecell[c]{\textbf{No}} \\ 
    \midrule
    \makecell[c]{Local SGD~\cite{li2020on}}& 
    \makecell[c]{Strongly Convex} 
    & \makecell[c]{$O(\frac{k^2}{NT})$} & \makecell[c]{$O(T)$} & \makecell[c]{Non-IID} & \makecell[c]{(1)} \\
    \makecell[c]{Local SGD~\cite{karimireddy2019scaffold}}\footnotemark[1] & 
    \makecell[c]{Strongly Convex} 
    & \makecell[c]{$O(\frac{\log{T}}{NT})$} & \makecell[c]{$O(N^\frac{1}{2} T^\frac{1}{2})$} & \makecell[c]{Non-IID} & \makecell[c]{No} \\
    \makecell[c]{SCAFFOLD~\cite{karimireddy2019scaffold}}\footnotemark[1]& 
    \makecell[c]{Strongly Convex} 
    & \makecell[c]{$O(\frac{\log{T}}{NT})$} & \makecell[c]{$O(\log{(NT)})$} & \makecell[c]{Non-IID} & \makecell[c]{No} \\
    \makecell[c]{\textbf{STL-SGD}} &
    \makecell[c]{\textbf{Strongly Convex}} 
    & \makecell[c]{\bm{$O(\frac{1}{NT})$}} & \makecell[c]{\bm{$O(N^{\frac{1}{2}} T^{\frac{1}{2}})$}} & \makecell[c]{\textbf{Non-IID}} & \makecell[c]{\textbf{No}}\\
    \midrule
    \makecell[c]{Local SGD~\cite{haddadpour2019local}}\footnotemark[2]& 
    \makecell[c]{Non-Convex+PL} 
    & \makecell[c]{$O(\frac{1}{NT})$} & \makecell[c]{$O(N^{\frac{1}{3}} T^{\frac{1}{3}})$} & \makecell[c]{IID} & \makecell[c]{No} \\
    \makecell[c]{\textbf{STL-SGD}} & 
    \makecell[c]{\textbf{Non-Convex+PL}} 
    & \makecell[c]{\bm{$O(\frac{1}{NT})$}} & \makecell[c]{\bm{$O(N \log{T})$}} & \makecell[c]{\textbf{IID}} & \makecell[c]{\textbf{No}} \\
    \midrule
    \makecell[c]{\textbf{STL-SGD}} & 
    \makecell[c]{\textbf{Non-Convex+PL}} 
    & \makecell[c]{\bm{$O(\frac{1}{NT})$}} & \makecell[c]{\bm{$O(N^{\frac{1}{2}} T^{\frac{1}{2}})$}} & \makecell[c]{\textbf{Non-IID}} & \makecell[c]{\textbf{No}} \\
    \midrule
    \makecell[c]{Local SGD~\cite{wang2018cooperative}} & 
    \makecell[c]{Non-Convex} 
    & \makecell[c]{$O(\frac{1}{\sqrt{NT}})$} & \makecell[c]{$O(N^{\frac{3}{2}} T^{\frac{1}{2}})$} & \makecell[c]{IID} & \makecell[c]{(1)} \\
    \makecell[c]{\textbf{STL-SGD}} & 
    \makecell[c]{\textbf{Non-Convex}} 
    & \makecell[c]{\bm{$O(\frac{1}{\sqrt{NT}})$}} & \makecell[c]{\bm{$O(N^{\frac{3}{2}} T^{\frac{1}{2}})$}} & \makecell[c]{\textbf{IID}} & \makecell[c]{\textbf{No}} \\
    \midrule
    \makecell[c]{Local SGD~\cite{shen2019faster}} & 
    \makecell[c]{Non-Convex} 
    & \makecell[c]{$O(\frac{1}{\sqrt{NT}})$} & \makecell[c]{$O(N^{\frac{3}{4}} T^{\frac{3}{4}})$} & \makecell[c]{Non-IID} & \makecell[c]{(2)} \\
    \makecell[c]{Local SGD~\cite{haddadpour2019convergence}} & 
    \makecell[c]{Non-Convex} 
    & \makecell[c]{$O(\frac{1}{\sqrt{NT}})$} & \makecell[c]{$O(N^{\frac{3}{2}} T^{\frac{1}{2}})$} & \makecell[c]{Non-IID} & \makecell[c]{(3)} \\
    \makecell[c]{SCAFFOLD~\cite{karimireddy2019scaffold}} & 
    \makecell[c]{Non-Convex} 
    & \makecell[c]{$O(\frac{1}{\sqrt{NT}})$} & \makecell[c]{$O(N^{\frac{1}{2}} T^{\frac{1}{2}})$} & \makecell[c]{Non-IID} & \makecell[c]{No} \\
    \makecell[c]{\textbf{STL-SGD}} & 
    \makecell[c]{\textbf{Non-Convex}} 
    & \makecell[c]{\bm{$O(\frac{1}{\sqrt{NT}})$}} & \makecell[c]{\bm{$O(N^{\frac{3}{4}} T^{\frac{3}{4}})$}} & \makecell[c]{\textbf{Non-IID}} & \makecell[c]{\textbf{No}} \\
    \bottomrule
    \end{tabular}
    }
    \end{center}
    \vskip -0.1in
    \end{table*}

  \section{Related Works}
  \label{relatedwork}
  \paragraph{
  Local SGD.}
  When the data distributions on clients are identical, 
  Local SGD is proved to achieve $O(\frac{1}{NT})$ convergence for strongly convex objectives~\cite{stich2018local} and $O(\frac{1}{\sqrt{NT}})$ convergence for non-convex objectives~\cite{wang2018cooperative} when the communication period $k$ satisfies $k \leq O(T^{\frac{1}{2}} / N^{\frac{3}{2}})$. 
  As demonstrated in these results, Local SGD achieves a linear speedup with the communication complexity $O(N^{\frac{3}{2}} T^{\frac{1}{2}})$ for both strongly convex and non-convex objectives in the IID case. 
  In addition, \cite{haddadpour2019local} justifies that $O(N^{\frac{1}{3}} T^{\frac{1}{3}})$ rounds of communication are sufficient to achieve 
  $O(\frac{1}{NT})$ convergence
  for objectives which satisfy the Polyak-\L ojasiewicz condition.
  On the other hand, for the Non-IID case, 
  Local SGD is proved with a
  $O(1/\sqrt{NT})$ convergence rate under 
  a communication complexity of $O(N^{\frac{3}{4}} T^{\frac{3}{4}})$ for non-convex objectives~\cite{yu2019parallel,shen2019faster}.
  Meanwhile, for strongly convex objectives, 
  a 
  suboptimal
  convergence rate of $O(\frac{k^2}{\mu NT})$~\cite{li2020on} 
  is
  obtained.
  Beyond that, when a small fixed learning rate is adopted, \cite{bayoumi2020tighter} and \cite{karimireddy2019scaffold} prove that the communication complexity of Local SGD is $O(N \log(NT))$ and $O(N^\frac{1}{2} T^\frac{1}{2})$ for the IID case and the Non-IID case respectively, at the cost of a suboptimal convergence rate $O(\frac{\log{T}}{NT})$. 
  For general non-convex objectives, \cite{haddadpour2019convergence} proves a lower communication complexity of $O(N^{\frac{3}{2}} T^{\frac{1}{2}})$ for the Non-IID case under the assumption of bounded gradient diversity.
  From the practical view, 
  \cite{zhang2016parallel} suggests
  to communicate more frequently in the beginning of the training, and \cite{haddadpour2019local} verifies
  that using a geometrically increasing period does not harm the convergence notably. 
  
  \paragraph{Stagewise Training.} For training both strongly convex and non-convex objectives, stagewisely decreasing the learning rate is widely adopted.
  Epoch-SGD~\cite{hazan2014beyond} and ASSG~\cite{xu2017stochastic} use SGD as their subalgorithm and geometrically decrease the learning rate stage by stage. 
  They are proved to achieve the optimal $O (1/T)$ convergence for stochastic strongly convex optimization. 
  For training neural networks, stagewisely decreasing the learning rate~\cite{krizhevsky2012imagenet,he2016deep} is a very important trick.
  From a theoretical aspect, stagewise SGD is proved with $O (1/\sqrt{T})$ convergence for both general and composite non-convex objectives~\cite{allen2018make,chen2018universal,davis2019proximally}, by adopting SGD to optimize a regularized objective at each stage and decreasing the learning rate linearly stage by stage.
  Stagewise training is also verified to achieve better testing error than general SGD~\cite{yuan2019stagewise}.

  \paragraph{Large Batch SGD (LB-SGD).} 
  SyncSGD with extremely large batch is proved to achieve a linear speedup with respect to the batch size~\cite{stich2019errorfeedback}. Nevertheless, \cite{jain2016parallelizing} shows that increasing the batch size does not help when the bias dominates the variance. It is also observed from practice that LB-SGD leads to a poor generalization~\cite{keskar2016largebatch,golmant2018computational,yin2017gradient}.
  \cite{yu2019computation} proposes CR-PSGD which increases the batch size geometrically step by step and proves that CR-PSGD achieves a linear speedup with $O(\log{T})$ communication complexity. However, after a large number of iterations, CR-PSGD essentially becomes GD and loses the benefit of SGD.
  
  \paragraph{Local SGD with Variance Reduction. } Recently, several techniques are proposed to reduce the communication complexity of Local SGD in the Non-IID case. 
  \cite{haddadpour2019trading} shows
  that using redundant data among clients yields lower communication complexity. One variant of Local SGD called VRL-SGD~\cite{liang2019variance} incorporates the variance reduction technique and is proved to achieve a $O(N^{\frac{3}{2}} T^{\frac{1}{2}})$ communication complexity for non-convex objectives. 
  SCAFFOLD~\cite{karimireddy2019scaffold} extends VRL-SGD by involving 
  two separate learning rates, and is proved to achieve  $O(\log{(NT)})$ and $O(N^{\frac{1}{2}} T^{\frac{1}{2}})$ communication complexities for strongly convex objectives and non-convex objectives respectively. 
  As SCAFFOLD adopts a small fixed learning rate, its convergence rate for strongly convex objectives is $O(\frac{\log{T}}{NT})$.
  Nevertheless, these methods are orthogonal to our study. Combining STL-SGD and variance reduction to get better performance for the Non-IID case exceeds the scope 
  of this paper.  
  
  Table~\ref{comparison} summarizes the comparison of Local SGD and its state-of-the-art extensions with STL-SGD. 
  For both strongly convex objectives and non-convex objectives which satisfy the PL condition, STL-SGD achieves the state-of-the-art communication complexity 
  while attaining the optimal convergence rate of $O(\frac{1}{NT})$.
  It is worth mentioning that Local SGD with momentum~\cite{yu2019linear1} or adaptive learning rate~\cite{reddi2020adaptive} are orthogonal to our study. 

\footnotetext[1]{Although these studies prove lower communication complexity, a suboptimal $O(\frac{\log{T}}{NT})$ convergence rate is proved due to the small fixed learning rate.}
\footnotetext[2]{The adaptive variant of Local SGD proposed in \cite{haddadpour2019local} has the same order of communication complexity as Local SGD.}

  \section{Preliminaries}
  \label{preliminary}
  \subsection{Notations and Definitions}
  Throughout the paper, we let $\| \cdot \|$ indicate the $\ell_2$ norm of a vector and $\langle \cdot, \cdot \rangle$ indicate the inner product of two vectors. The set $\{1,2, \cdots, n\}$ is denoted as $[n]$. We use $x^*$ to represent the optimal solution of (\ref{basic_object}).
  $\nabla f$ represents the gradient of $f$. 
  $\mathbb{E}$ indicates a full expectation with respect to all the randomness in the algorithm (the stochastic gradients sampled in all iterations and the randomness in return). 
  
  The data distributions on different clients may not be identical. 
  To quantify the difference of distributions, we define 
  $\zeta_f^* := \frac{1}{N} \sum_{i=1}^N \|\nabla f_i(x^*)\|^2 = \frac{1}{N} \sum_{i=1}^N \|\nabla f_i(x^*) - \nabla f(x^*)\|^2$, which represents the variance of gradients among clients 
  at 
  $x^*$.
  Some literatures assume that the variance of gradients among clients is bounded by a constant $\zeta^2$~\cite{shen2019faster} or the norm of stochastic gradients is bounded by a constant $G^2$~\cite{yu2019parallel,li2020on}. 
  Note that both $\zeta^2$ and $G^2$ are larger than $\zeta_f^*$. When the data distributions are identical, we have $ \| \nabla f_i(x^*) \|^2 = 0$, thus it holds that $\zeta_f^* = 0$.
  
  All proofs are deffered to the appendix.
  To state the convergence of algorithms for solving (\ref{basic_object}), we introduce some commonly used definitions~\cite{chen2018universal,haddadpour2019local}.
  
  \begin{definition}[$\rho$-weakly convex] \label{def:3}
    A non-convex function $f(x)$ is $\rho$-weakly convex ($\rho > 0$) 
    if
    \begin{equation}
    f(x) \geq f(y) + \langle \nabla f(y), x - y \rangle - \frac{\rho}{2} \| x - y \|^2, \forall x, y \in R^d.
    \nonumber
    \end{equation}
  \end{definition}
  \begin{definition}[$\mu$-Polyak-\L ojasiewicz (PL)]\label{def:4}
    A function $f(x)$ satisfies the $\mu$-PL condition ($\mu > 0$) 
    if
    \begin{equation}
      2\mu (f(x) - f(x^*)) \leq \| \nabla f(x) \|^2, \forall x \in R^d.
      \nonumber
    \end{equation}
  \end{definition}

  \subsection{Assumptions}
  Throughout this paper, we 
  make the following assumptions, all of which are commonly used and basic assumptions~\cite{stich2018local,yu2019parallel,li2020on,chen2018universal,allen2018make}.
  \begin{assumption}\label{assu1}
  $f_i(x)$ is $L$-smooth in terms of $i \in [N]$ for every $x \in R^d$:
  \begin{equation*}
    \| \nabla f_i(x) - \nabla f_i(y) \| \leq L \|x - y \|, \forall x, y \in R^d, i \in [N].
  \end{equation*}
  \end{assumption}
  
  \begin{assumption}\label{assu2}
    There exists a constant $\sigma$ such that 
    \begin{equation*}
        \mathbb{E}_{\xi \sim \mathcal{D}_i} \| \nabla f(x, \xi) - \nabla f_i (x) \|^2 \leq \sigma^2, \forall x \in R^d, \forall i \in [N].
    \end{equation*}
  \end{assumption}
  
  \begin{assumption}\label{assu3}
   If the objective function is non-convex, we assume it is $\rho$-weakly convex.
  \end{assumption}
  \begin{remark}
  Note that if $f(x)$ is $L$-smooth, it is $L$-weakly convex. This is because Assumption~\ref{assu1} implies $-\frac{L}{2}\| x - y \|^2 \leq f(x) - f(y) - \langle \nabla f(y), x-y \rangle \leq \frac{L}{2}\| x-y \|^2$~\cite{nesterov2018lectures}. Therefore, for an $L$-smooth function, we can immediately get that the weakly-convex parameter $\rho$ satisfies $0<\rho \leq L$.
  \end{remark}

  \subsection{Review: Synchronous SGD with Periodically Averaging (Local SGD)}
  To alleviate the high communication cost in SyncSGD, the periodically averaging technique is proposed~\cite{stich2018local,yu2019parallel}. Instead of averaging models in all clients at every iteration, Local SGD lets clients update their models locally for $k$ iterations, then one communication is conducted to average the local models to make them consistent. Specifically, the update rule of Local SGD is
  {\small
  \begin{equation}
  x_{t}^i = 
    \begin{cases}
        \frac{1}{N} \sum_{j=1}^N ( x_{t-1}^j -  \eta \nabla f(x_{t-1}^j, \xi_{t-1}^j)), & {\textrm{if}~t~\%~k~=~0},\\  
      x_{t-1}^i - \eta \nabla f(x_{t-1}^i, \xi_{t-1}^i),  &{\textrm{else}},
    \end{cases}
    \nonumber
  \end{equation}
  }where $x_t^i$ is the local model in client $i$ at iteration $t$.
  Therefore, when each client conducts $T$ iterations, the total number of communications is $T/k$. The complete procedure of Local SGD is summarized in Algorithm~\ref{Local_SGD}.
  Different from previous studies~\cite{mcmahan2017communication,stich2018local,yu2019parallel}, Algorithm~\ref{Local_SGD} returns $\tilde{x} = \frac{1}{N} \sum_{i=1}^N  x_t^i$ for a randomly chosen $t \in \{ 0, 1, \cdots, T-1\}$.
  In practice, we can determine $t$ at first to avoid redundant iterations.
  
  \begin{algorithm}[!tb]
    \caption{Local-SGD($f$, $x_0$, $\eta$, $T$, $k$)}
    \label{Local_SGD}
    \hspace*{0.02in}{\bf Initialize:} $x_0^i = x_0, \forall i \in [N].$
    \begin{algorithmic}[1]
    \FOR{$t = 1, ..., T$}
      \STATE \underline{Client $C_i$ does}:
      \STATE Uniformly sample a mini-batch $\xi_{t-1}^i \in \mathcal{D}_i$ and calculate a stochastic gradient $\nabla f_i(x_{t-1}^i, \xi_{t-1}^i)$.
      \IF{$t$ divides $k$}
          \STATE Communicate with other clients and update: $x_{t}^i = \sum_{j=1}^N \frac{1}{N} (x_{t-1}^j - \eta \nabla f(x_{t-1}^j, \xi_{t-1}^j))$.
      \ELSE
          \STATE Update locally: $x_{t}^i = x_{t-1}^i - \eta \nabla f_i(x_{t-1}^i, \xi_{t-1}^i)$.
      \ENDIF
    \ENDFOR
    \STATE \textbf{return} $\tilde{x} = \frac{1}{N} \sum_{i=1}^N  x_t^i$ for the randomly chosen $t \in \{ 0, 1, \cdots, T-1\}$.
    \end{algorithmic}
  \end{algorithm}
  
  Although several studies have analysed the convergence of Local SGD, they assume that the objective $f(x)$ is $\mu$-strongly convex or non-convex. 
  \cite{khaled2019first} focuses
   on general convex objectives while they use the full gradient descent.
  Besides, most of the existing analysis 
  relies on some stronger assumptions, including bounded gradient norm (i.e., $\| \nabla f_i(x, \xi) \|^2 \leq G^2$)~\cite{stich2018local,li2020on} or bounded variance of gradients among clients~\cite{shen2019faster}. Here, we give a basic convergence result of Local SGD for the general convex objectives without these assumptions.
  
  \begin{theorem} \label{theorem1}
  Suppose Assumptions~\ref{assu1} and \ref{assu2} hold, $f(x)$ is convex and $\eta \leq \frac{1}{6L}$.
  If we set $k \leq  \min\{ \frac{1}{6 \eta L N}, \frac{1}{9 \eta L}\} $ and $k \leq  \min\{ \frac{\sigma}{\sqrt{6 \eta L N(\sigma^2 + 4\zeta_f^*)}}, \frac{1}{9 \eta L} \} $ for the IID case and the Non-IID case respectively, we have
  \begin{equation}\label{theo1:0}
    \mathbb{E} f(\tilde{x}) - f(x^*) \leq \frac{3 \| x_0 - x^*\|^2}{4 \eta T} + \frac{\eta \sigma^2}{N}.
  \end{equation}
  \end{theorem}
  \begin{remark}
    If we set $\eta = \sqrt{\frac{N}{T}}$, we have $\mathbb{E}f(\tilde{x}) - f(x^*) \leq \frac{\| x_0 - x^* \|^2 + \sigma^2}{\sqrt{NT}}$, which is consistent with the result of mini-batch SGD~\cite{dekel2012optimal}. 
  \end{remark}

  \section{Local SGD with Stagewise Communication Period}
  \label{method}
  To further reduce the communication complexity, we propose \textit{ST}agewise \textit{L}ocal \textit{SGD} (STL-SGD) in this section with the following features. 
  \begin{itemize}
    \item At the beginning, STL-SGD employs Algorithm~\ref{Local_SGD} as a subalgorithm in each stage.
    \item Instead of using a small fixed learning rate or a gradually decreasing learning rate (e.g. $\frac{\eta_1}{1+\alpha t}$), STL-SGD adopts a stagewisely adaptive scheme. The learning rate is fixed at first, 
    and decreased stage by stage.
    \item The communication periods are increased stagewisely.
  \end{itemize}
  We propose two variants of STL-SGD for strongly convex and non-convex problems, respectively.
   
  \subsection{STL-SGD for Strongly Convex Problems} \label{STL-SGD_convex_sec}
  \begin{algorithm}[!tb]
    \caption{$\text{STL-SGD}^{sc}$($f$, $x_1$, $\eta_1$, $T_1$, $k_1$)}
    \label{STL-SGD_sc}
      \begin{algorithmic}[1]
      \FOR{$s = 1, 2, ..., S$}
        \STATE $x_{s+1}$ = Local-SGD($f$, 
              $x_s$, $\eta_s$, $T_s$, $\max\{\lfloor k_s \rfloor, 1\}$).
        \STATE Set $\eta_{s+1} = \frac{ \eta_{s} }{2}$, $T_{s+1} = 2 T_s$ and $$k_{s+1} = 
            \begin{cases} \sqrt{2} k_s, ~~~{\rm Non\text{-}IID~case},\\
            2 k_s,~~~~~~{\rm IID~case}.
            \end{cases} $$
      \ENDFOR
      \STATE \textbf{return} $x_{S+1}$.
    \end{algorithmic}
  \end{algorithm}
  In this subsection, we propose the STL-SGD algorithm for strongly convex problems, which is denoted as $\text{STL-SGD}^{sc}$ and summarized in Algorithm~\ref{STL-SGD_sc}. At each stage, the learning rate is decreased exponentially. In the meantime, the number of iterations and the communication period are increased exponentially. Specifically, at the $s$-th stage, we set $\eta_s = \frac{\eta_{s-1}}{2}$ and $T_{s} = 2 T_{s-1}$. The communication period $k_s$ is set as $k_s = 2 k_{s-1}$ and $k_s = \sqrt{2} k_{s-1} $ for the IID case and the Non-IID case respectively.   
  
  Below, let $x_{s}$ denote the initial point of the $s$-th stage. Theorem~\ref{theorem2} establishes the convergence rate of $\text{STL-SGD}^{sc}$.
  
  \begin{theorem}\label{theorem2}
    Suppose $f(x)$ is $\mu$-strongly convex. 
    Let $\eta_1 \leq \frac{1}{6L}$ and $T_1 \eta_1 = \frac{6}{\mu}$. We set $k_1 = \min\{  \frac{1}{6 \eta_1 L N}, \frac{1}{9 \eta_1 L} \} $ and $k_1 = \min\{\frac{\sigma}{\sqrt{6 \eta_1 L N(\sigma^2 + 4\zeta_f)}}, \frac{1}{9 \eta_1 L} \}$ for the IID case and the Non-IID case respectively. Under Assumptions~\ref{assu1} and \ref{assu2}, when the number of stages satisfies $S \geq \log(\frac{N(f(x_0) - f(x^*))}{\eta_1 \sigma^2}) + 2$, we have the following result for Algorithm~\ref{STL-SGD_sc}:
    \begin{equation}\label{theo2:0}
      \mathbb{E} f(x_{S+1}) - f(x^*) \leq \frac{9 \eta_1 \sigma^2}{2^S N}.
    \end{equation}
    Defining $T := T_1 + T_2 + \cdots + T_S$, we have 
    \begin{equation}
      \mathbb{E} f(x_{S+1}) - f(x^*) \leq 
      O\left(\frac{1}{NT}\right).
    \end{equation}
  \end{theorem}
  \begin{remark}\label{remark3}
  Theorem~\ref{theorem2} claims the following properties of $\text{STL-SGD}^{sc}$:
  \begin{itemize}
    \item \textbf{ Linear Speedup. } 
    To reach a solution $x_{S+1}$ with $ \mathbb{E} f(x_{S+1}) - f(x^*) \leq \epsilon $, the number of iterations is $ O (\frac{1}{N\epsilon} )$, which indicates a linear speedup.
    \item \textbf{Communication Complexity for the Non-IID Case.}
    For the Non-IID case, we set $k_{s+1} = \sqrt{2} k_s$ for Algorithm~\ref{STL-SGD_sc}. Therefore, the total communication complexity is $\frac{T_1}{k_1} + \cdots + \frac{T_S}{k_S} = \frac{T_1}{k_1}(1 + 2^{\frac{1}{2}} + \cdots + 2^{\frac{s-1}{2}}) = O( \frac{T_1}{k_1} \cdot (\frac{T}{T_1})^{\frac{1}{2}}) = O (N^{\frac{1}{2}} T^{\frac{1}{2}})$, where the last equality holds because $\frac{T_1^{\frac{1}{2}}}{k_1} = O (\sqrt{T_1 \eta_1 N}) = O (N^{\frac{1}{2}})$.
    \item \textbf{Communication Complexity for the IID Case.} If the data distributions on different clients are identical, we set $k_{s+1} = 2 k_s$ for Algorithm~\ref{STL-SGD_sc}. Thus, the total communication complexity is $\frac{T_1}{k_1} + \cdots + \frac{T_S}{k_S} = S \frac{T_1}{k_1} = O (N \log{T})$. 
  \end{itemize}
  \end{remark}
  
  \subsection{STL-SGD for Non-Convex Problems}\label{STL-SGD_non_convex_sec}
  \begin{algorithm}[!tb]
    \caption{$\text{STL-SGD}^{nc}$($f$, $x_1$, $\eta_1$, $T_1$, $k_1$)}
    \label{STL-SGD_nc}
      \begin{algorithmic}[1]
      \FOR{$s = 1, 2, ..., S$}
        \STATE Let $f_{x_{s}}^{\gamma}(x) = f(x) + 
              \frac{1}{2\gamma} \|x - x_s \|^2$.
        \STATE $x_{s+1}$ = Local-SGD($f_{x_{s}}^{\gamma}$, 
              $x_s$, $\eta_s$, $T_s$, $\max\{\lfloor k_s \rfloor, 1\}$).
        \STATE \textbf{Option~1:} Set $\eta_{s+1} = \frac{ \eta_{s} }{2}$, 
                $T_{s+1} = 2 T_s$  and $$k_{s+1} = \begin{cases}
                \sqrt{2} k_s, ~~~~~~~{\rm Non\text{-}IID~case},\\
                2 k_s,~~~~~~~~~~{\rm IID~~case}.
                \end{cases} $$
        \STATE \textbf{Option~2:} Set $\eta_{s+1} = \frac{\eta_1}{s+1}$, $T_{s+1}=(s+1)T_1$ and $$k_{s+1} = \begin{cases}
          \sqrt{s+1} k_1, ~~~~~~~{\rm Non\text{-}IID~case},\\
          (s+1) k_1,~~~~~~~{\rm IID~case}.
          \end{cases} $$
      \ENDFOR
      \STATE \textbf{return} $x_{S+1}$.
    \end{algorithmic}
    \end{algorithm}
  In this subsection, we proceed to propose the variant of STL-SGD algorithm for non-convex problems ($\text{STL-SGD}^{nc}$). Different from Algorithm~\ref{STL-SGD_sc}, which optimizes a fixed objective during all stages, $\text{STL-SGD}^{nc}$ changes the objective once a stage is finished. Specifically, in the $s$-th stage, the objective is a regularized problem $f_{x_{s}}^{\gamma} = f(x) + \frac{1}{2\gamma} \|x - x_s \|^2$, where $x_s$ is the initial point of the $s$-th stage and $\gamma$ is a constant that satisfies $\gamma < \rho^{-1}$. $f_{x_{s}}^{\gamma}(x)$ is guaranteed to be convex due to the $\rho$-weak convexity of $f(x)$. 
  In this way, the theoretical property of Algorithm~\ref{Local_SGD} under convex settings still holds in each stage of $\text{STL-SGD}^{nc}$.
  Other parameters are set in two different ways (\textbf{Option~1} and \textbf{Option~2}) for non-convex objectives satisfying the PL condition and otherwise, which are detailed in Algorithm~\ref{STL-SGD_nc}.

  In \textbf{Option~1}, we set $\eta_s$, $T_s$ and $k_s$ in the same way as in Algorithm~\ref{STL-SGD_sc}. Here we analyse the theoretical property of $\text{STL-SGD}^{nc}$ with \textbf{Option~1} for non-convex objectives that satisfy the PL condition. 
  \begin{theorem}\label{theorem4}
    Assume $f(x)$ satisfies the PL condition defined in 
    Definition~\ref{def:4} with constant $\mu$. 
    Suppose Assumptions~\ref{assu1}, \ref{assu2} and \ref{assu3} hold and $f(x)$ is weakly convex with constant $\rho \leq \frac{\mu}{16}$. Let $\eta_1 \leq \frac{1}{12L_{\gamma}}$, $T_1 \eta_1 = \frac{6}{\rho}$. Set $k_1 = \min\{  \frac{1}{6 \eta_1 L_\gamma N}, \frac{1}{9 \eta_1 L_\gamma} \} $ and $k_1 = \min\{\frac{\sigma}{\sqrt{6 \eta_1 L_\gamma N(\sigma^2 + 4\zeta_f)}}, \frac{1}{9 \eta_1 L_\gamma} \} $ for the IID case and the Non-IID case respectively.
    When the number of stages satisfies $S \geq \log{\frac{N\left( f(x_0) - f(x^*) \right)}{\eta_1 \sigma^2}} + 2$,
    Algorithm~\ref{STL-SGD_nc} with \textbf{Option~1} returns a solution $x_{S+1}$ such that
    \begin{equation}
      \mathbb{E} f(x_{S+1}) - f(x^*) \leq 
      O\left(\frac{1}{NT}\right),
    \end{equation}
    where $T = T_1 + T_2 + \cdots + T_S$.
  \end{theorem}
  \begin{remark}
    As the result of Theorem~\ref{theorem4} is the same as that of Theorem~\ref{theorem2}, properties stated in Remark~\ref{remark3} all hold here.
  \end{remark}
  
  \textbf{Option~2} is employed for the non-convex objectives which do not satisfy the PL condition. Instead of increasing the communication period geometrically as in \textbf{Option~1} of Algorithm~\ref{STL-SGD_nc}, we let it increase in a linear manner, i.e., $k_{s} = s k_1$. Meanwhile, we increase the stage length linearly, that is $T_s = s T_1$, while keeping $T_s \eta_s$ a constant.
  \begin{theorem}\label{theorem5}
    Suppose Assumptions~\ref{assu1}, \ref{assu2} and \ref{assu3} hold. Let $\eta_1 \leq \frac{1}{6L_{\gamma}}$ and $T_1 \eta_1 = \frac{3}{\rho}$. Set $k_1 = \min\{  \frac{1}{6 \eta_1 L N}, \frac{1}{9 \eta_1 L} \} $ and $k_1 = \min\{  \frac{\sigma}{\sqrt{6 \eta_1 L N(\sigma^2 + 4\zeta_f)}}, \frac{1}{9 \eta_1 L} \} $ for the IID case and the Non-IID case respectively. Algorithm~\ref{STL-SGD_nc} with \textbf{Option~2} guarantees that 
    \begin{eqnarray}
      \mathbb{E} \|\nabla f(x_s) \|^2 \leq O \left( \frac{1}{\sqrt{N T}} \right),
    \end{eqnarray}
    where $s$ is randomly sampled from $\{1,2,\cdots,S\}$ with probability $p_s = \frac{s}{1+2+\cdots+S}$.
  \end{theorem}

  \begin{figure*}[t]
    \centering
    \subfigure{
    \begin{minipage}[t]{0.225\linewidth}
    \centering
    \includegraphics[width=1.1\textwidth] {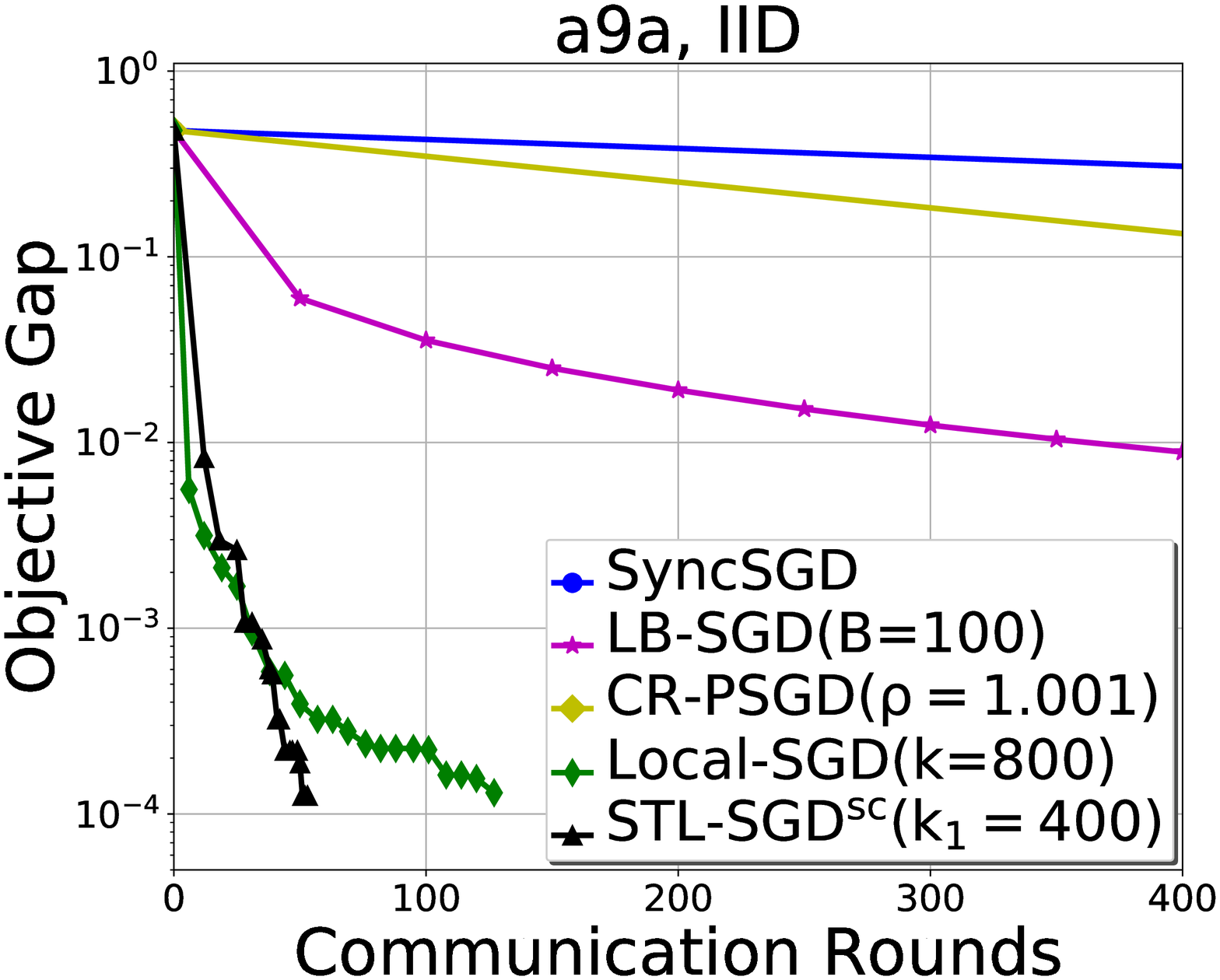}
    \end{minipage}
    }
    \subfigure{
    \begin{minipage}[t]{0.225\linewidth}
    \centering
    \includegraphics[width=1.1\textwidth] {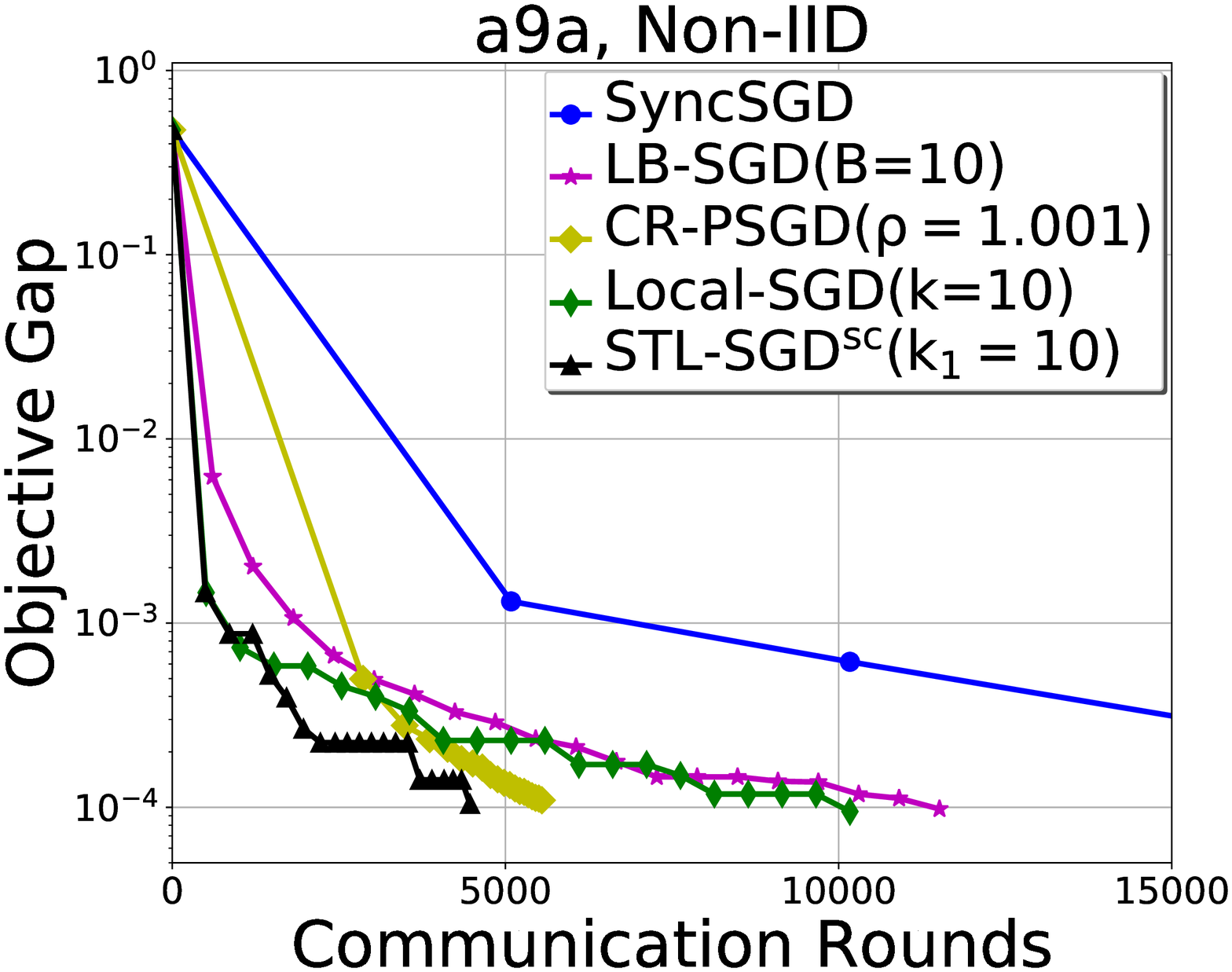}
    \end{minipage}
    }
    \subfigure{
    \begin{minipage}[t]{0.225\linewidth}
    \centering
    \includegraphics[width=1.1\textwidth] {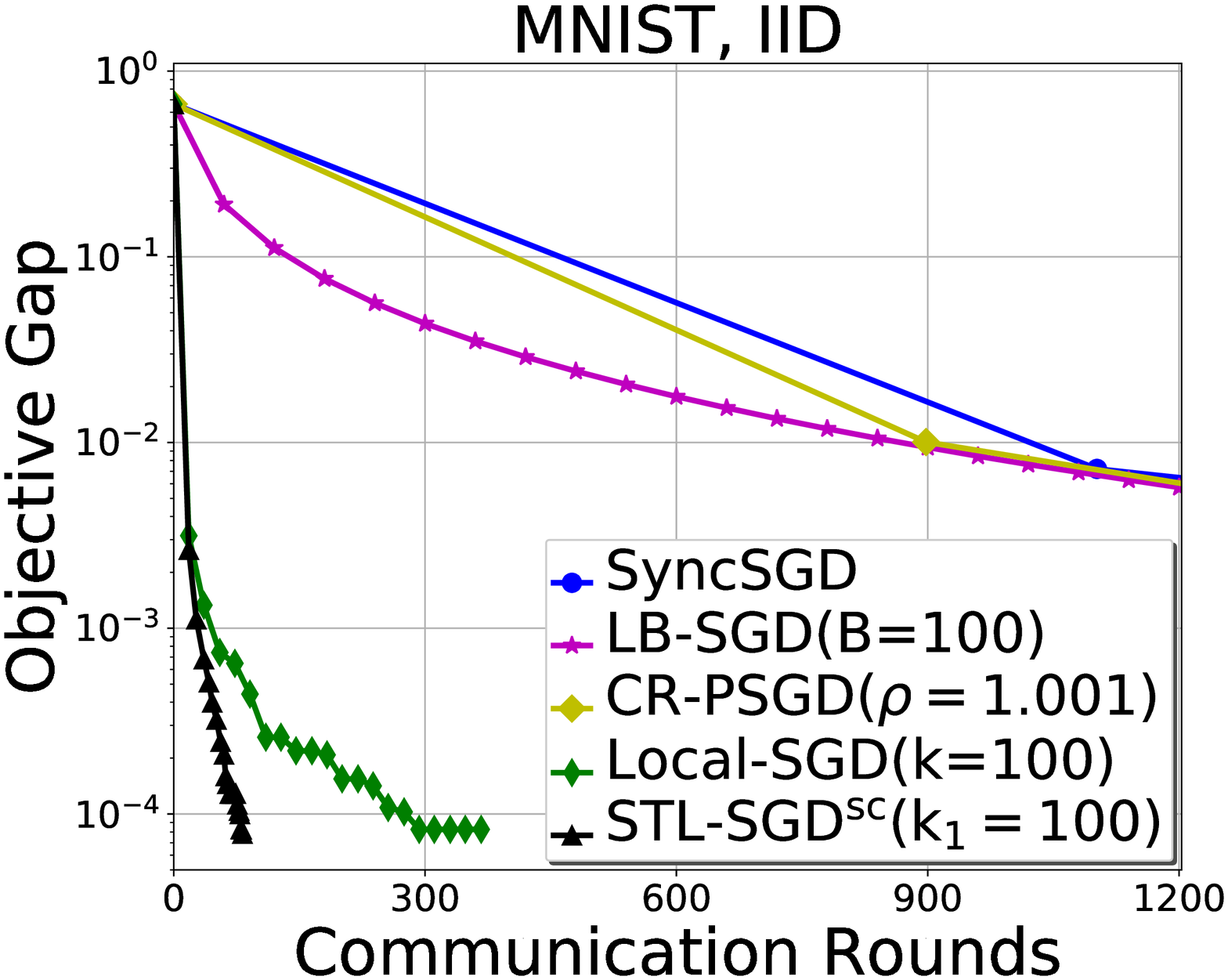}
    \end{minipage}
    }
    \subfigure{
    \begin{minipage}[t]{0.225\linewidth}
    \centering
    \includegraphics[width=1.1\textwidth] {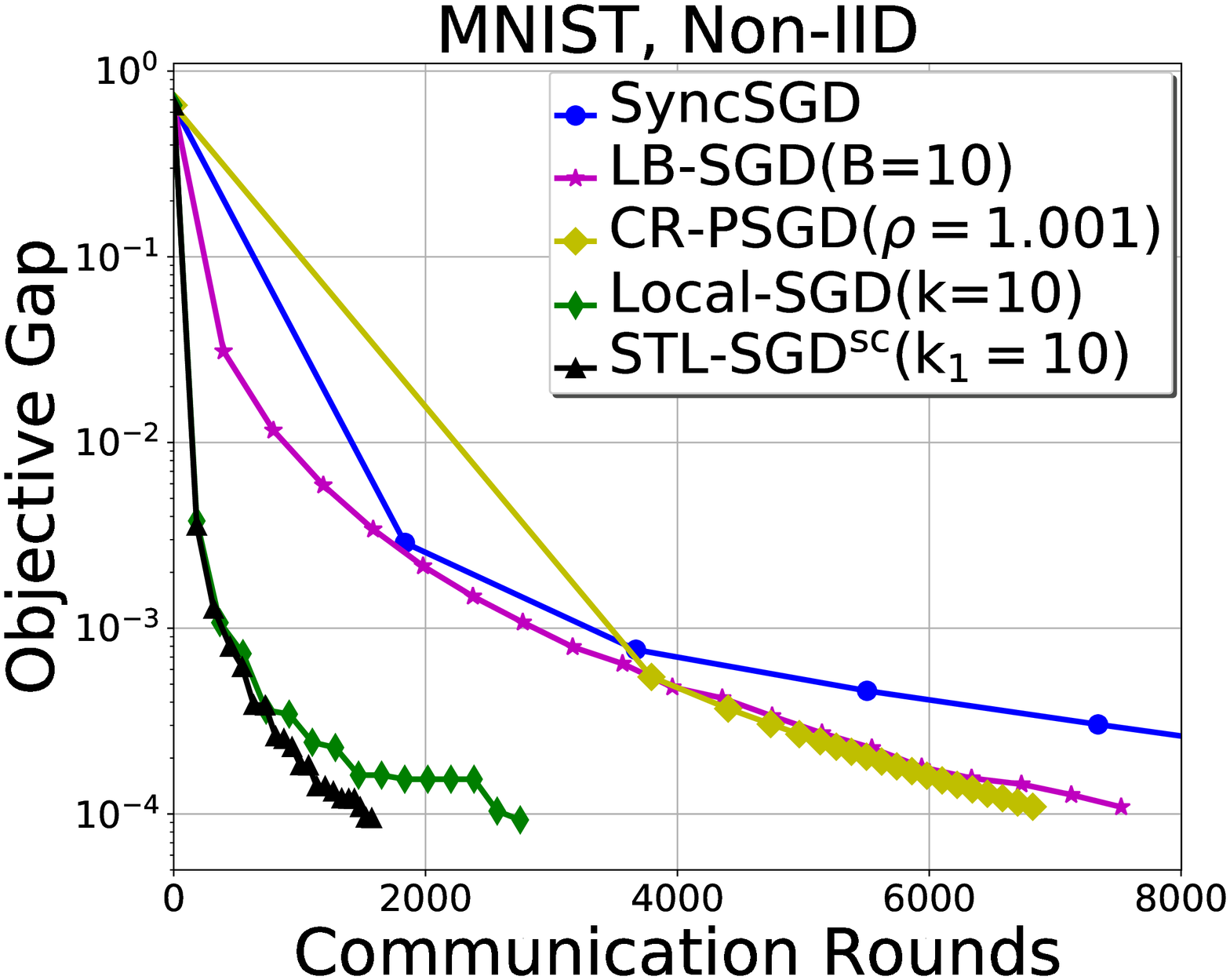}
    \end{minipage}
    }
    \caption{Training objective gap $f(x) - f(x^*)$ w.r.t the communication rounds for logistic regression on $\mathrm{a9a}$ and $\mathrm{MNIST}$.
    }
    \label{Compare_convex}
  \end{figure*}
  
  \begin{table*}[!t]
    \caption{Communication rounds to reach $10^{-4}$ objective gap in convex problems. We also show the speedup of these algorithms compared with SyncSGD.}
    \label{comparison_comm_convex}
    \begin{center}
    \begin{tabular}{l l l l l}
      \toprule
      Algorithms & a9a (IID) & a9a (Non-IID)  & MNIST (IID) & MNIST (Non-IID) \\
      \midrule
      SyncSGD & 100683 ($1\times$)    & 90513 ($1\times$)   & 32664 ($1\times$)  & 22021 ($1\times$)\\
      LB-SGD & 7620 ~~~~($13.2\times$)    & 12221 ($7.4\times$) & 7011 ~~($4.7\times$) & 7740 ~~($2.8\times$)\\
      CR-PSGD & 5434 ~~~~($18.5\times$)   & 5772 ~~($15.7\times$) & 6788 ~~($4.8\times$) & 7029 ~~($3.1\times$)\\
      Local-SGD & 184 ~~~~~~($547.2\times$) & 10068 ($9.0\times$) & 289 ~~~~($113.0\times$)& 2642 ~~($8.3\times$)\\
      $\text{STL-SGD}^{sc}$ & \textbf{{61 ~~~~~~~~(\bm{$ 1650.5\times$}) }}  & \textbf{4417 ~~(\bm{$20.5\times$}) }& \textbf{79 ~~~~~~(\bm{$413.5\times$}) }& \textbf{1518 ~~(\bm{$14.5\times$}) }\\
      \bottomrule
    \end{tabular}
    \end{center}
  \end{table*}

  \begin{remark}
    $\text{STL-SGD}^{nc}$ with \textbf{Option~2} has the following properties:
    \begin{itemize}
      \item \textbf{Linear Speedup:} To achieve $\mathbb{E} \| \nabla f(x_S) \|^2 \leq \epsilon$, the total number of iterations when $N$ clients are used is $O (\frac{1}{N\epsilon^2})$, which shows a linear speedup.
      \item \textbf{Communication Complexity for the Non-IID case:} Algorithm~\ref{STL-SGD_nc} with \textbf{Option~2} sets $k_{s} = \sqrt{s} k_1$. Thus, the communication complexity is $\frac{T_1}{k_1} + \frac{T_2}{k_2} + \cdots + \frac{T_S}{k_S} = \frac{T_1}{k_1} (1 + \sqrt{2} + \cdots + \sqrt{S}) = O (\frac{T_1}{k_1} (\frac{T}{T_1})^{\frac{3}{4}} )=O(N^{\frac{3}{4}}T^{\frac{3}{4}})$.
      \item \textbf{Communication Complexity for the IID case:} As $k_{s} = s k_1$, the communication complexity is $\frac{T_1}{k_1} + \frac{T_2}{k_2} + \cdots + \frac{T_S}{k_S} = \frac{T_1}{k_1}S = O(\frac{T_1}{k_1}(\frac{T}{T_1})^{\frac{1}{2}} ) = O \left( N^{\frac{3}{2}} T^{\frac{1}{2}} \right)$.
    \end{itemize}
  \end{remark}

  \section{Experiments}
  \label{experiments}
  We validate the performance of the proposed STL-SGD algorithm with experiments on both convex 
  and non-convex problems. For each type of problems, we conduct experiments 
  for both the IID case and the Non-IID case. Experiments are conducted on a machine with 8 Nvidia Geforce GTX 1080Ti GPUs and 2 Xeon(R) Platinum 8153 CPUs.

  To simulate the Non-IID scenarios, 
  we divide the training data and make the distributions of classes different among clients. Similar to the setting in \cite{karimireddy2019scaffold}, at first, we randomly take $s\%$ i.i.d. data from the training set and divide them equally to each client. For the remaining data, we sort them according to their classes and then assign them to the clients in order. In our experiments, we set $s=50$ for convex problems and $s=0$ for non-convex problems.

  We compare $\text{STL-SGD}$ with SyncSGD, LB-SGD, CR-PSGD~\cite{yu2019computation} and Local SGD~\cite{stich2018local}. We show the comparison of these algorithms in terms of the communication rounds. 
  The investigation regarding convergence is included in the appendix, which validates that STL-SGD can achieve similar convergence rate as SyncSGD. 

  \subsection{Convex Problems}
  We consider the binary classification problem with logistic regression, i.e.,
  \begin{equation}\label{logistic_regression}
    \min_{\theta\in R^d} \frac{1}{n} \sum_{i=1}^n \log(1 + \exp(-y_i x_i^T \theta)) + \frac{\lambda}{2} \| \theta \|^2,
  \end{equation}
  where $(x_i, y_i), i \in [n]$ constitute a set of training examples, and $\lambda$ is the regularization parameter. It is notable that (\ref{logistic_regression}) is 
  strongly convex 
  when $\lambda > 0$, and we set $\lambda = 1/n$. 
  We take two datasets $\mathrm{a9a}$ and $\mathrm{MNIST}$
  from the libsvm website\footnote[3]{https://www.csie.ntu.edu.tw/~cjlin/libsvmtools/datasets/}.
  $\mathrm{a9a}$ has $32,561$ examples and $123$ features.
  For $\mathrm{MNIST}$, we sample a subset with $11,791$ examples and $784$ features from two classes (4 and 9). 
  Experiments are implemented on 32 clients and communication is handled with MPI\footnote[4]{https://www.open-mpi.org/}.
  
  SyncSGD, LB-SGD and Local SGD are implemented with the decreasing learning rate $\eta_t = \frac{\eta_1}{1+\alpha t}$ as suggested in \cite{stich2018local,li2020on} and we tune $\alpha$ in $\{10^{-2}, 10^{-3}, 10^{-4}\}$ for the best performance.
  For $\text{STL-SGD}^{sc}$, we set $\eta_1 T_1 = \frac{1}{\lambda}$. The initial learning rate for all algorithms is tuned in $\{N, N/10, N/100\}$. The communication period $k$ and the batch size $B$ for LB-SGD are tuned in $\{100, 200, 400, 800, 1600\}$ for the IID case, and $\{10, 20, 40, 80, 160\}$ for the Non-IID case. The scaling factor of batch size $\rho$ for CR-PSGD is tuned in $\{1.001, 1.01, 1.1\}$. We report the largest $k$, $B$ and $\rho$ which do not sacrifice the convergence for all algorithms.
  
  \begin{figure*}[!t]
    \centering
    \subfigure{
      \begin{minipage}[t]{0.225\linewidth}
      \centering
      \includegraphics[width=1.1\textwidth] {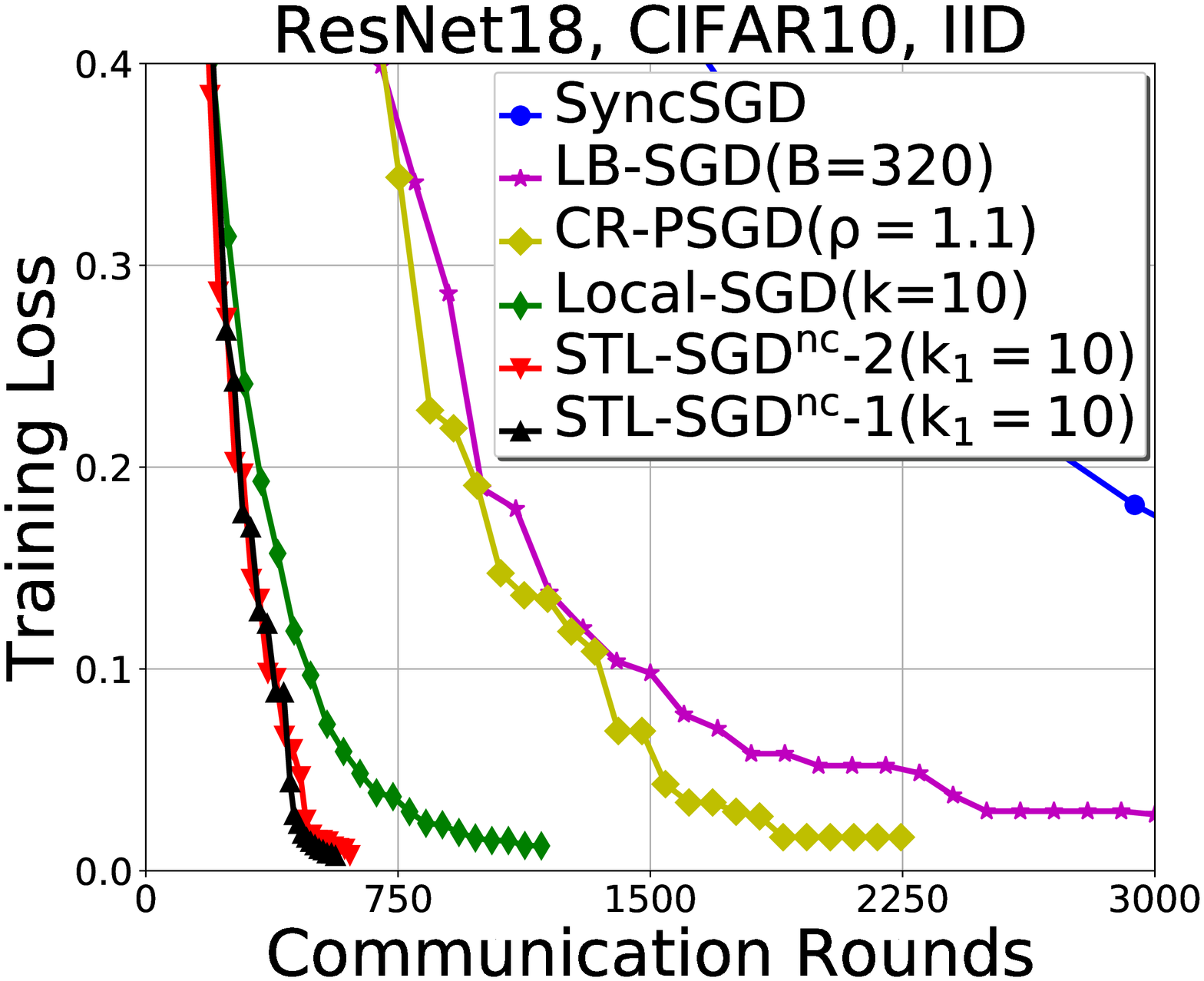}
      \end{minipage}
      }
    \subfigure{
      \begin{minipage}[t]{0.225\linewidth}
      \centering
      \includegraphics[width=1.1\textwidth] {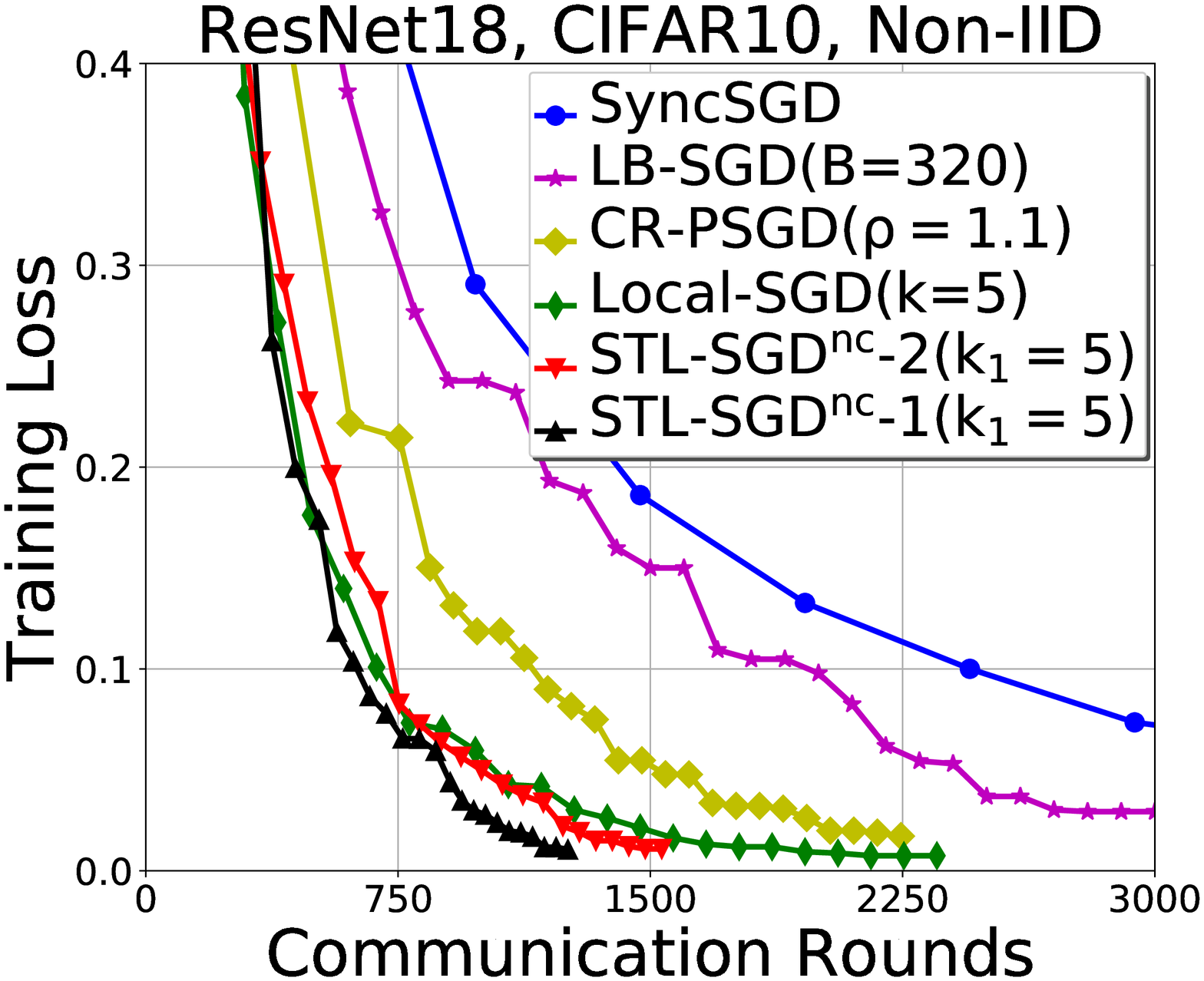}
      \end{minipage}
    }
    \subfigure{
      \begin{minipage}[t]{0.225\linewidth}
      \centering
      \includegraphics[width=1.1\textwidth] {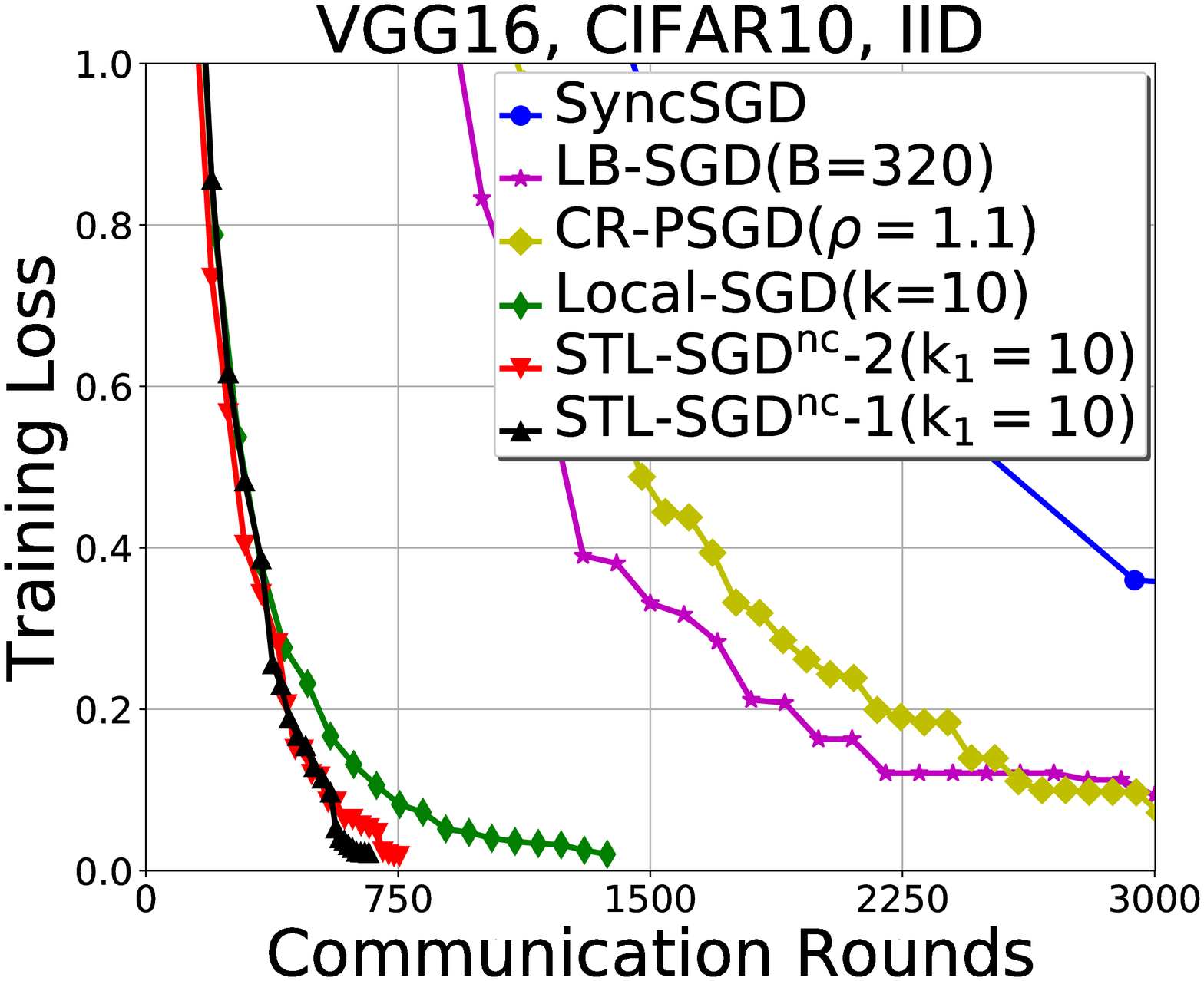}
      \end{minipage}
      }
    \subfigure{
      \begin{minipage}[t]{0.225\linewidth}
      \centering
      \includegraphics[width=1.1\textwidth] {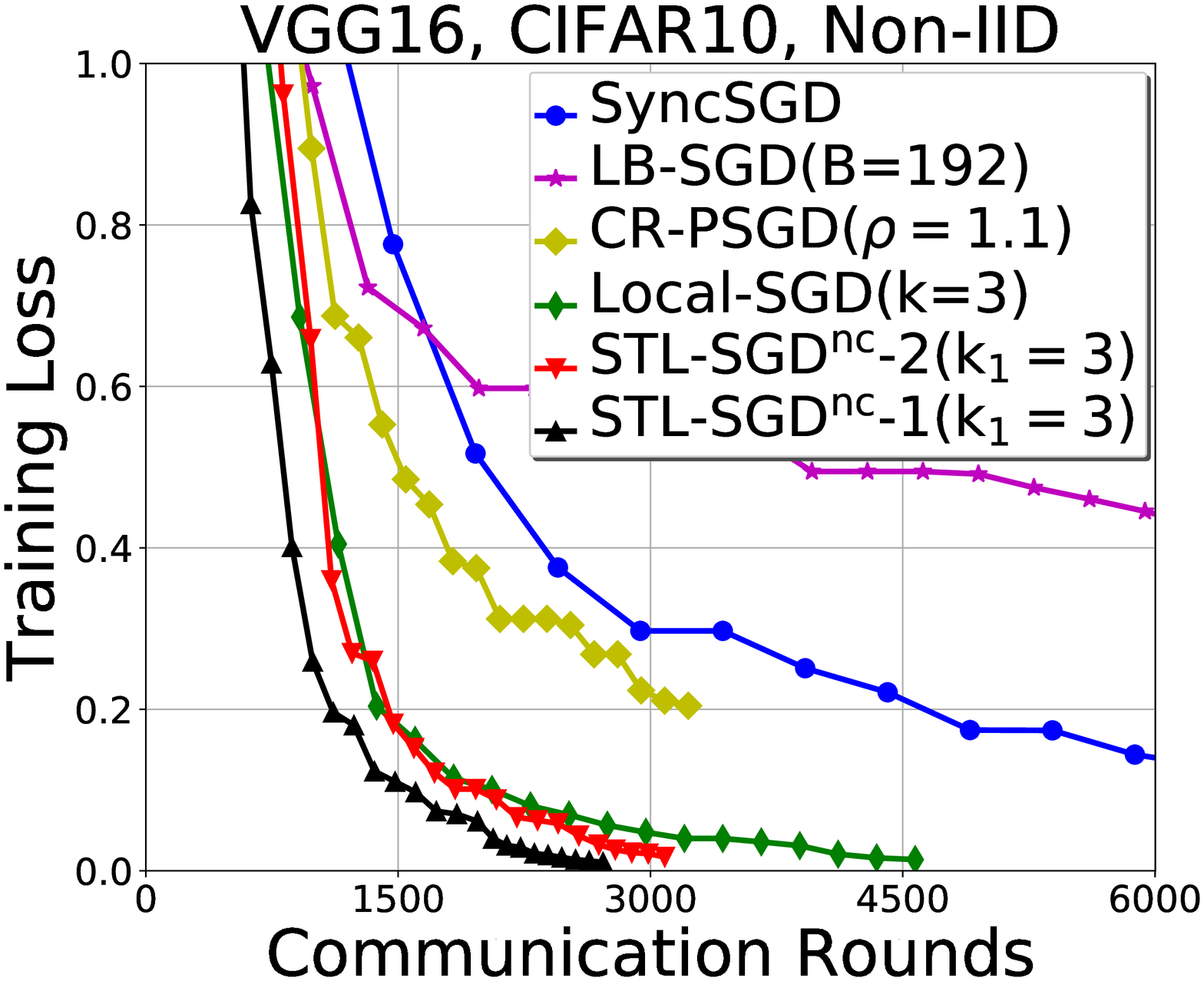}
      \end{minipage} 
    }
    \caption{Training loss w.r.t the communication rounds for ResNet18 and VGG16 on CIFAR10.
    }
    \label{Compare_nonconvex}
  \end{figure*}
  
  \begin{table*}[!t]
    \caption{Communication rounds to reach 99\% training accuracy in non-convex problems. We run all algorithms for 200 epochs, where an epoch indicates one pass of the dataset. LB-SGD and CR-PSGD can not achieve 99\% training accuracy on the VGG16 neural network until the end of training.}
    \label{comparison_comm_nonconvex}
    \begin{center}
    \begin{tabular}{l l l l l}
      \toprule
      Algorithms & ResNet18 (IID) & ResNet18 (Non-IID)  & VGG16 (IID) & VGG16 (Non-IID) \\
      \midrule
      SyncSGD & 7644 ($1\times$)    & 5390 ($1\times$)   & 13622 ($1\times$)  & 15092 ($1\times$)\\
      LB-SGD & 3000 ($2.5\times$)    & 3180 ($1.7\times$) & $-$ ~~~~~~~($-$) & $-$ ~~~~~~~($-$)\\
      CR-PSGD & 1797 ($4.3\times$)   & 1937 ($2.8\times$) & $-$ ~~~~~~~($-$) & $-$ ~~~~~~~($-$)\\
      Local-SGD & 755 ~~($10.1\times$) & 1235 ($4.4\times$) & 1245 ~~($10.9\times$)& 3986 ~~($3.8\times$)\\
      $\text{STL-SGD}^{nc}$-2 & \textbf{{470 ~~(\bm{$ 16.3\times$}) }}  & \textbf{1158 (\bm{$4.7\times$}) }& \textbf{696 ~~~~(\bm{$19.6\times$}) }& \textbf{2732 ~~(\bm{$5.5\times$}) }\\
      $\text{STL-SGD}^{nc}$-1 & \textbf{{434 ~~(\bm{$ 17.6\times$}) }}  & \textbf{954 ~~(\bm{$5.6\times$}) }& \textbf{602 ~~~~(\bm{$22.6\times$}) }& \textbf{2179 ~~(\bm{$6.9\times$}) }\\
      \bottomrule
    \end{tabular}
    \end{center}
  \end{table*}

  Figure~\ref{Compare_convex} shows the objective gap $f(x) - f(x^*)$ with regard to 
  the communication rounds.
  We can observe that $\text{STL-SGD}^{sc}$ converges with the fewest communication rounds for both the IID case and the Non-IID case.
  Although the initial communication period of $\text{STL-SGD}^{sc}$ may need to be set smaller than Local SGD in the the IID case, the total number of communication rounds of $\text{STL-SGD}^{sc}$ is still significantly lower, which validates that the communication complexity of $\text{STL-SGD}^{sc}$ is much lower than Local SGD.
  As shown in Table~\ref{comparison_comm_convex}, to achieve $10^{-4}$ objective gap, the communication rounds of $\text{STL-SGD}^{sc}$ is almost 1.7-3 times fewer than Local SGD.

  \subsection{Non-Convex Problems}
  We train ResNet18~\cite{he2016deep} 
  and VGG16~\cite{simonyan2014very} 
  on the $\mathrm{CIFAR10}$~\cite{krizhevsky2009learning} dataset, which includes a training set of 50,000 examples from 10 classes. 
  8 clients are used in total.

  For our proposed algorithm, we denote $\text{STL-SGD}^{nc}$ with \textbf{Option~1} and \textbf{Option~2} as $\text{STL-SGD}^{nc}$-1 and $\text{STL-SGD}^{nc}$-2 respectively. The learning rates of SyncSGD, LB-SGD, CR-PSGD and Local-SGD are all set fixed as suggested in their convergence theory~\cite{ghadimi2013stochastic,yu2019computation,yu2019parallel}.
  The initial learning rate for all algorithms is tuned in $\{N/10, N/100, N/1000\}$. 
  The basic batch size at each client is 64.
  The first stage length of $\text{STL-SGD}^{nc}$ is tuned in $\{20, 40, 60\}$ epochs. 
  The parameter $\gamma$ in $\text{STL-SGD}^{nc}$ is tuned in $\{ 10^0, 10^2, 10^4\}$.
  We tune the communication period $k$ in $\{3, 5, 10, 20\}$ and the batch size $B$ for LB-SGD in $\{192, 320, 640, 1280\}$. For ease of implementation, we increase the batch size in CR-PSGD with $B = \rho B$ once an epoch is finished, and $\rho$ is tuned in $\{1.1, 1.2, 1.3\}$. $B$ stops growing when it exceeds $512$ as suggested in \cite{yu2019computation}. We show the largest $k$ and $B$ which can maintain the same convergence rate as SyncSGD for all algorithms.
  
  The experimental results of training loss regarding communication rounds are presented in Figure~\ref{Compare_nonconvex} and the communication rounds to achieve 99\% training accuracy for all algorithms are shown in Table~\ref{comparison_comm_nonconvex}. 
  As can be seen, $\text{STL-SGD}^{nc}$-1 and $\text{STL-SGD}^{nc}$-2 converge with much fewer communications than other algorithms.
  In spite of the same order of communication complexity as Local SGD, the performance of $\text{STL-SGD}^{nc}$-2 is better as the benefit of the negative relevance between the learning rate and the communication period.
  $\text{STL-SGD}^{nc}$-1 converges with the fewest number of communications, as it uses a geometrically increasing communication period.

  \section{Conclusion}
  We propose STL-SGD, which adopts a stagewisely increasing communication period to reduce the communication complexity. 
  Two variants of STL-SGD ($\text{STL-SGD}^{sc}$ and $\text{STL-SGD}^{nc}$) are provided for strongly convex objectives and non-convex objectives respectively. Theoretically, we prove that: (i) STL-SGD maintains the convergence rate and linear speedup as SyncSGD; (ii) when the objective is strongly convex or satisfies the PL condition, 
  while attaining the optimal convergence rate $O(\frac{1}{NT})$,
  STL-SGD achieves the state-of-the-art communication complexity; (iii) when the objective is general non-convex, STL-SGD has the same communication complexity as Local SGD, while being more consistent with practical tricks. Experiments on both convex and non-convex problems demonstrate the effectiveness of the proposed algorithm.
  

\section*{Aknowledgement}
This research was supported by the National Natural Science
Foundation of China (61673364) and Anhui Provincial
Natural Science Foundation (2008085J31). We would like to
thank the Information Science Laboratory Center of USTC for
the hardware and software services. We also gratefully acknowledge Xianfeng Liang from USTC for his valuable discussion.

\bibliography{stlsgd}

\onecolumn
\newpage
\appendix

\section{More About Experiments}

\subsection{Experimental Results for Validating the Convergence Rate}

\begin{figure*}[h]
  \centering
  \subfigure{
  \begin{minipage}[t]{0.23\linewidth}
  \centering
  \includegraphics[width=1.1\textwidth] {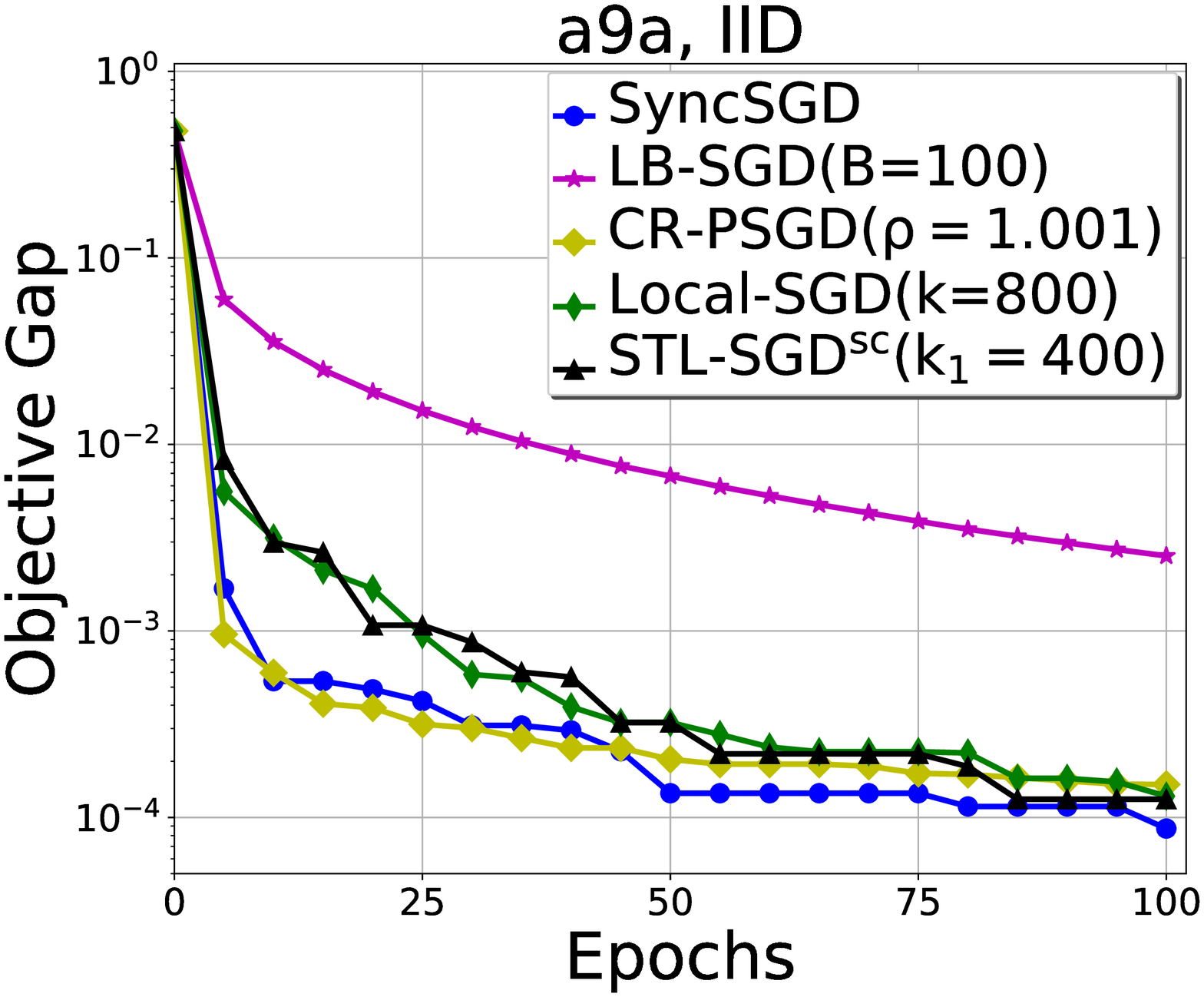}
  \end{minipage}
  }
  \subfigure{
  \begin{minipage}[t]{0.23\linewidth}
  \centering
  \includegraphics[width=1.1\textwidth] {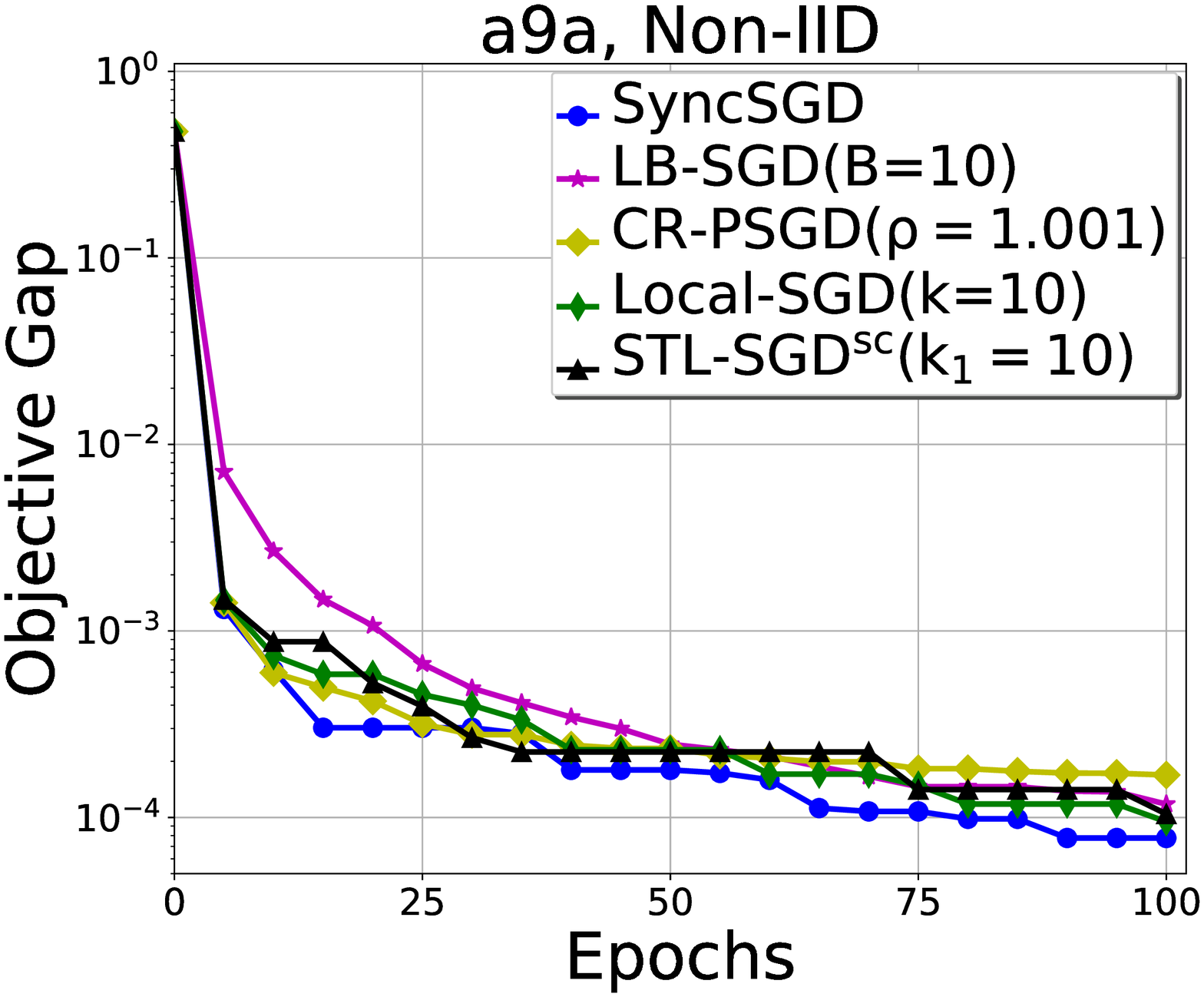}
  \end{minipage}
  }
  \subfigure{
    \begin{minipage}[t]{0.23\linewidth}
    \centering
    \includegraphics[width=1.1\textwidth] {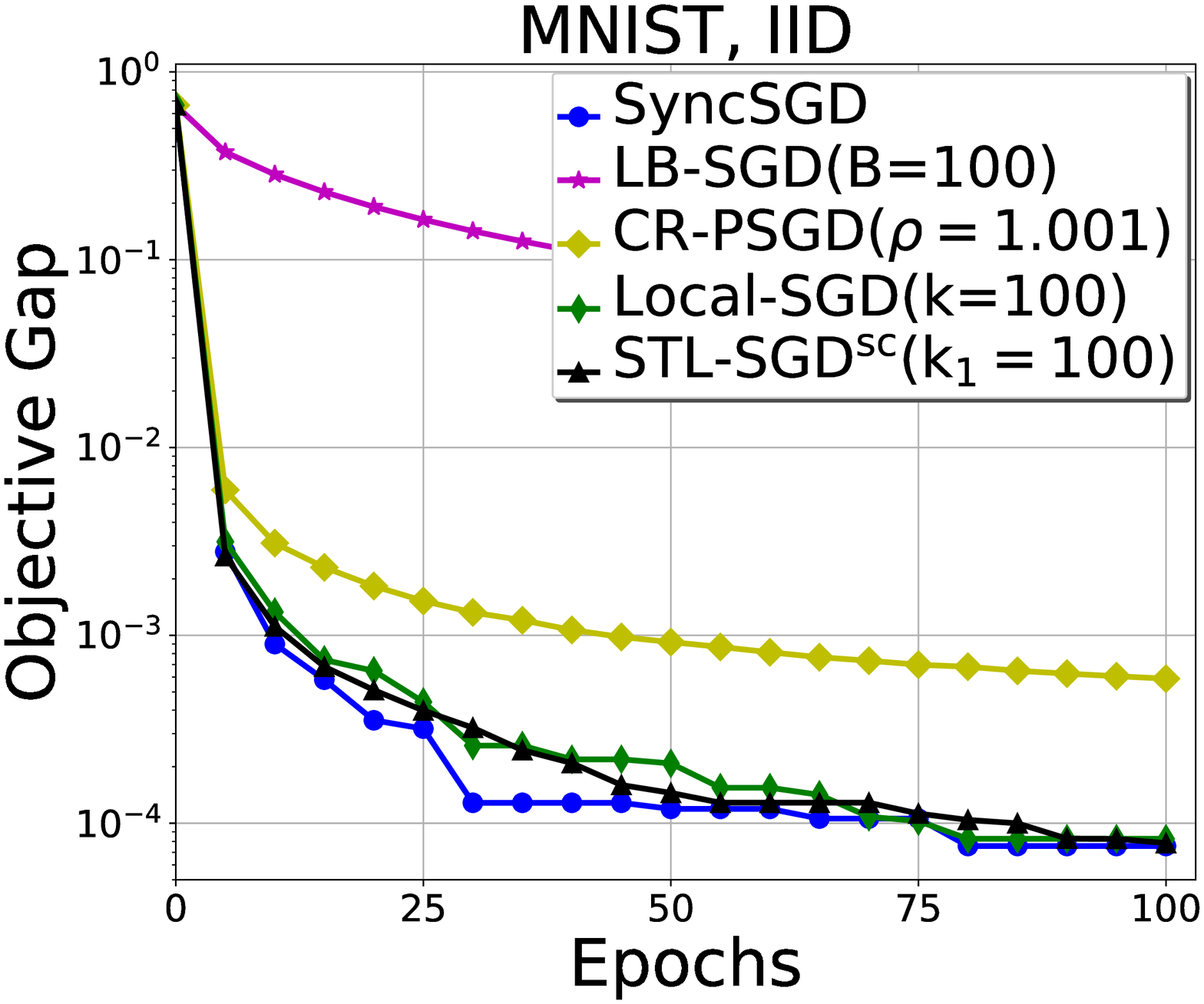}
    \end{minipage}
    }
    \subfigure{
    \begin{minipage}[t]{0.23\linewidth}
    \centering
    \includegraphics[width=1.1\textwidth] {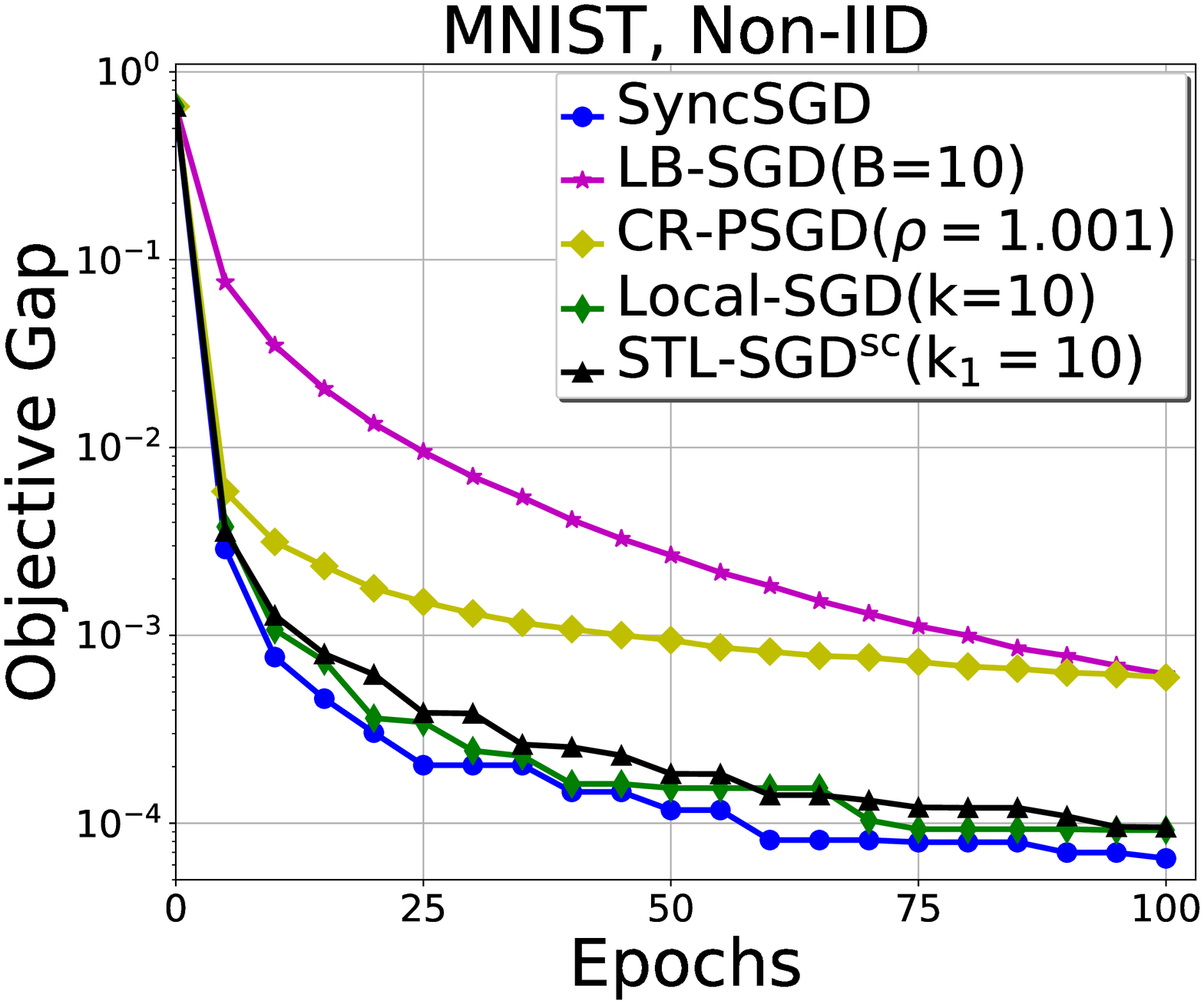}
    \end{minipage}
    }
  \caption{Training objective gap $f(x) - f(x^*)$ w.r.t epochs for logistic regression on $\mathrm{a9a}$ and $\textrm{MNIST}$ datasets. 
  }
  \label{Compare_convex_conv}
\end{figure*}

\begin{figure*}[h]
  \centering
  \subfigure{
    \begin{minipage}[t]{0.23\linewidth}
    \centering
    \includegraphics[width=1.1\textwidth] {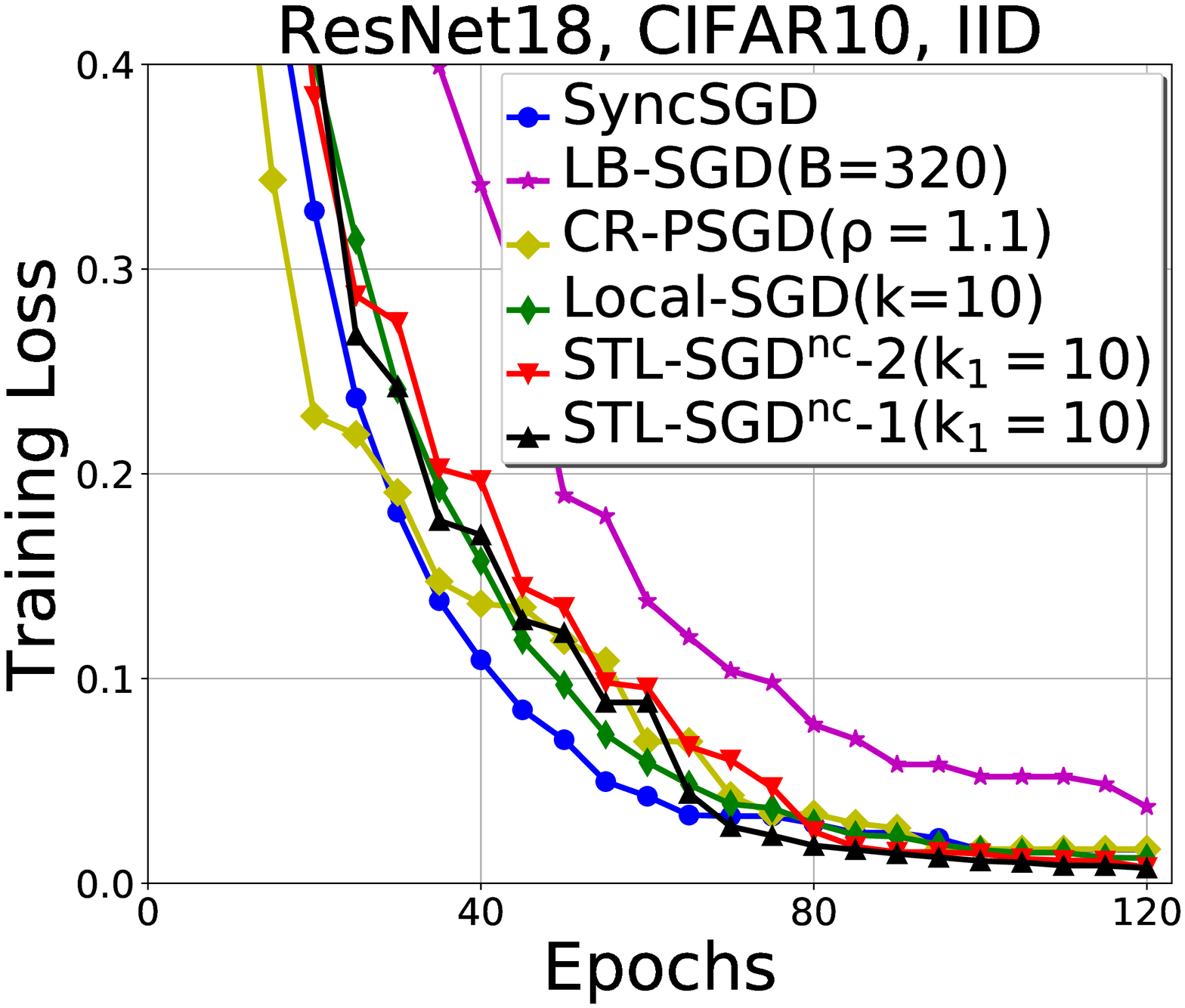}
    \end{minipage}
  }
  \subfigure{
    \begin{minipage}[t]{0.23\linewidth}
    \centering
    \includegraphics[width=1.1\textwidth] {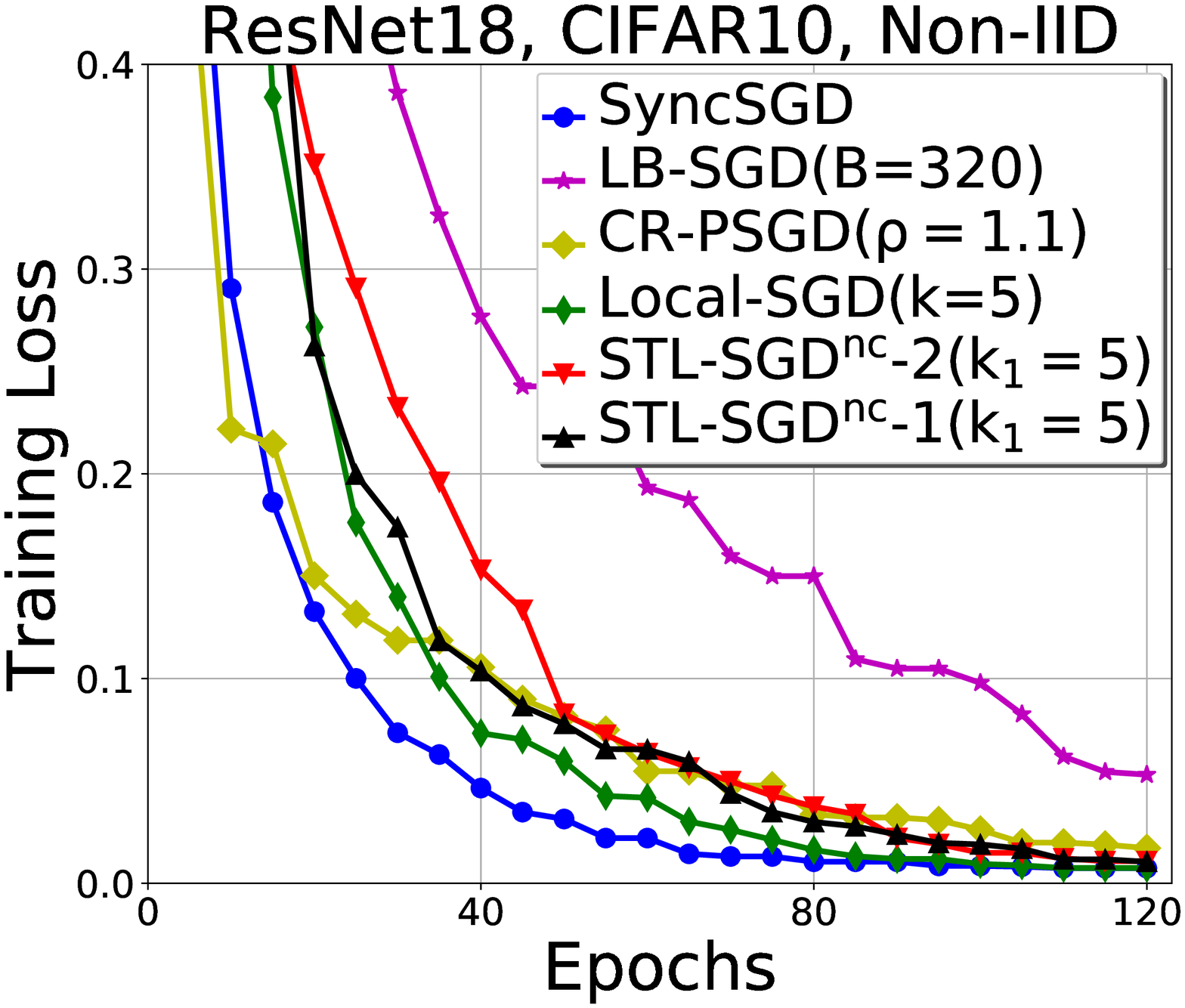}
    \end{minipage}
  }
  \subfigure{
    \begin{minipage}[t]{0.23\linewidth}
    \centering
    \includegraphics[width=1.1\textwidth] {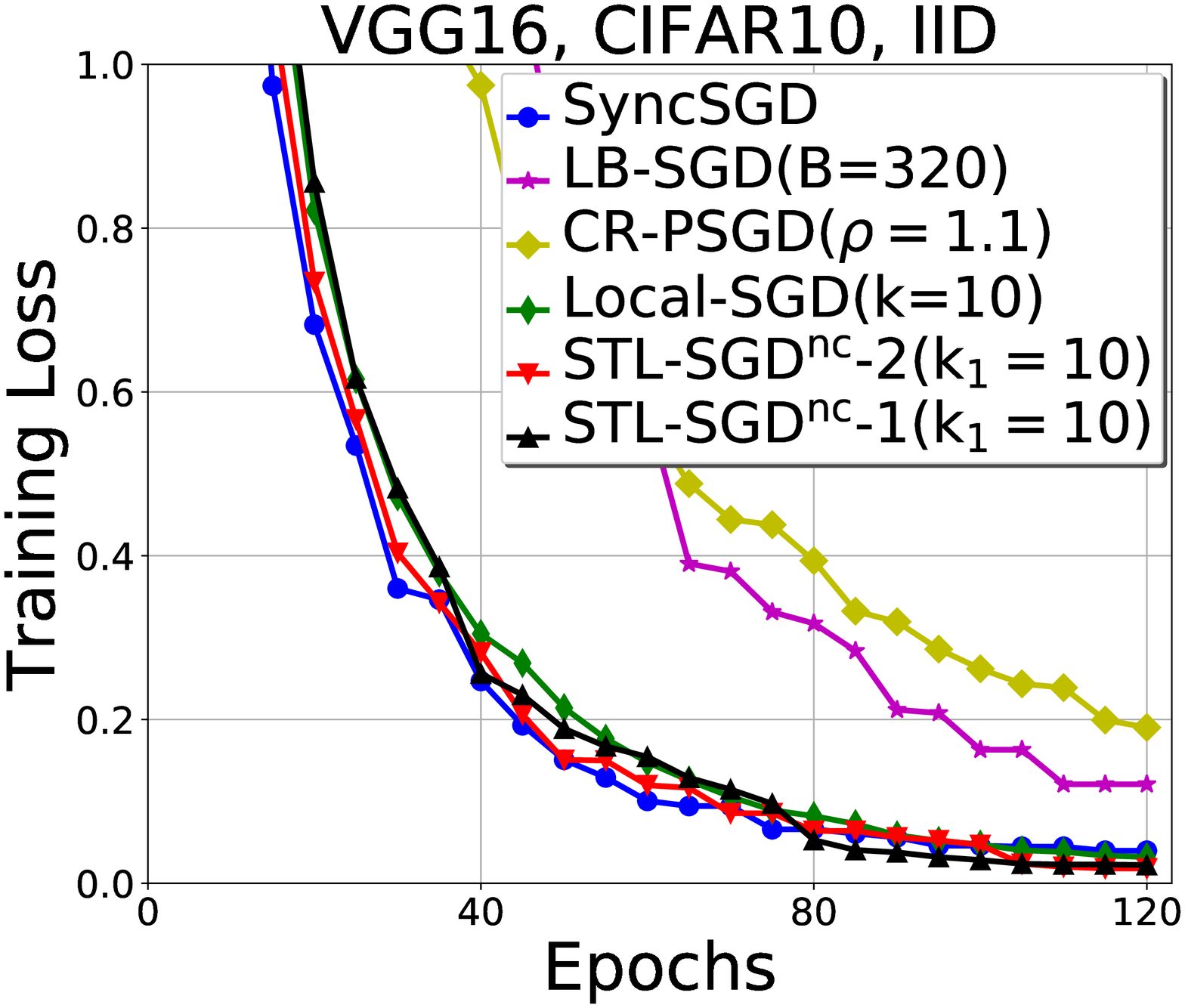}
    \end{minipage}
  }
  \subfigure{
    \begin{minipage}[t]{0.23\linewidth}
    \centering
    \includegraphics[width=1.1\textwidth] {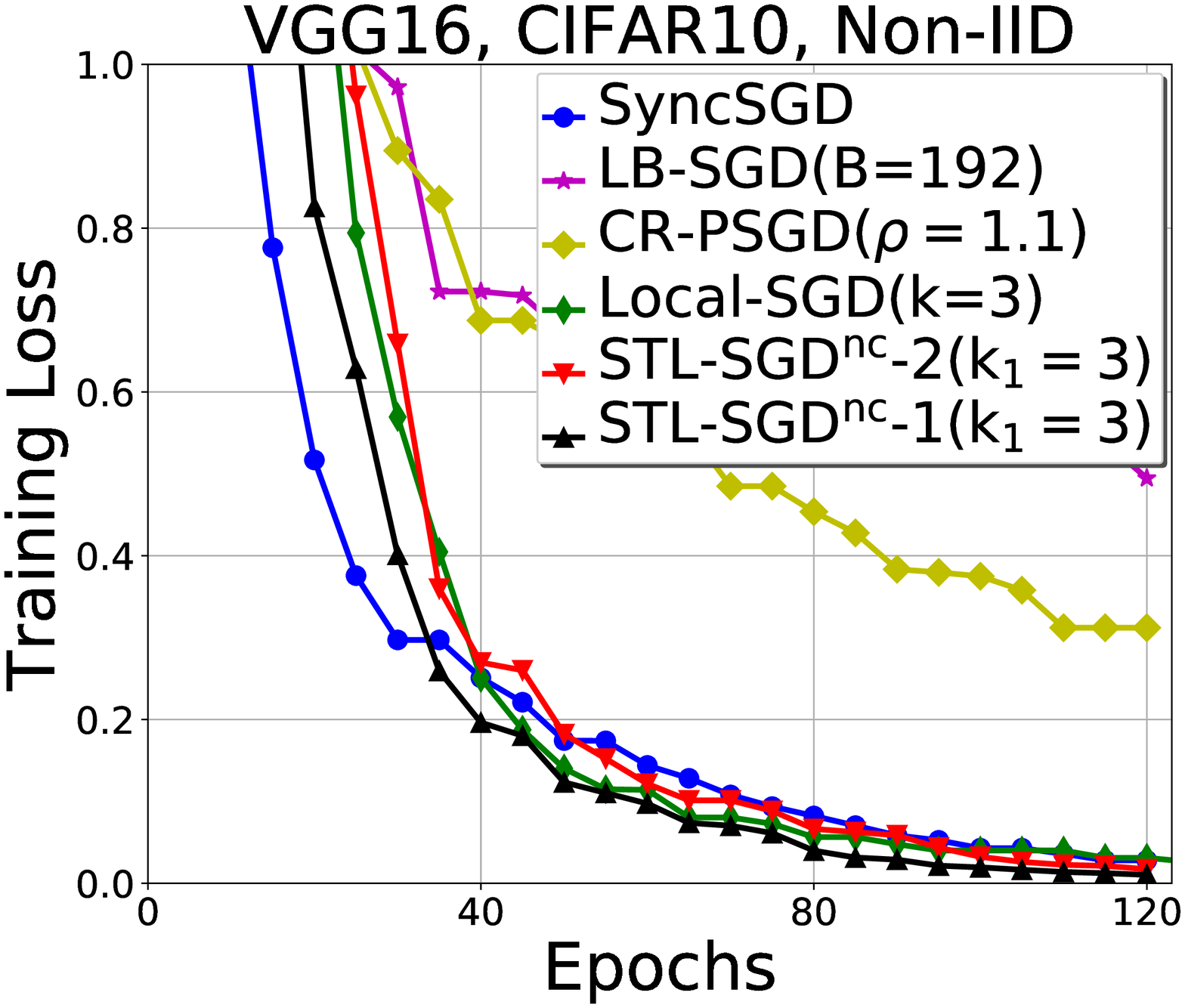}
    \end{minipage}
  }

  \caption{Training loss w.r.t epochs for ResNet18 and VGG16 on CIFAR10 dataset. 
  }
  \label{Compare_nonconvex_conv}
\end{figure*}

In this subsection, we supplement the experimental results not included in the main paper. The rules for turning the hyper-parameters are presented in Subsection Experiments and we turn all hyper-parameters to make all algorithms to achieve the best convergence speed. 
We present the experimental results of the training loss with regard to the epochs in this subsection.
The results for strongly convex 
objectives 
and non-convex objectives are shown in Figure~\ref{Compare_convex_conv} and Figure~\ref{Compare_nonconvex_conv} respectively.

From the theoretical perspective, STL-SGD, CR-PSGD and Local SGD can maintain the same convergence rate with SyncSGD: $O(\frac{1}{NT})$ for strongly convex objectives and $O(\frac{1}{\sqrt{NT}})$ for non-convex objectives. 
As shown in Figure~\ref{Compare_convex_conv} and Figure~\ref{Compare_nonconvex_conv}, when the hyper-parameters are set properly, the convergence speed of the above algorithms is  similar. 
STL-SGD and Local SGD may converge slowly in the beginning, but they match SyncSGD when the number of iterations is relatively large, which is consistent with our theory in Theorem~2 and Theorem~3 that the number of stages can not be too small.
Although LB-SGD is theoretically justified to achieve a linear speedup with respect to the batch size, it can not maintain the convergence of mini-batch SGD (or SyncSGD) when the batch size $B$ gets large. The reason could be that the bias dominates the variance as discussed in (Jain et al. 2016).

\section[C]{Proofs for Results in Section~Preliminaries}
In this section, we first present some lemmas, then give the proof for Theorem~1.
\subsection{Some Basic Lemmas}
We bound the norm of the difference between gradients with the Bregman divergence $\mathcal{D}_f(x, y) := f(x) - f(y) - \langle \nabla f(y), x-y \rangle$ for a smooth and convex function.
\begin{lemma}\label{smooth_basic}
  Suppose $f(x)$ is $L$-smooth and convex. The following inequality holds:
  \begin{equation}
    \| \nabla f(x) - \nabla f(y) \|^2 \leq 2L \mathcal{D}_f(x, y).\nonumber
  \end{equation}
\end{lemma}
\begin{proof}
This Lemma is identical to Theorem 2.1.5 (2.1.10) in (Nesterov 2018), which is a basic property of smooth and convex functions.
\end{proof}

For ease of analysis, we define $\hat{x}_t$ as the average of the local models, i.e., $\hat{x}_t = \frac{1}{N} \sum_{i=1}^N x_t^i$.  According to the update rule in Algorithm~1, we have
\begin{equation}
  \hat{x}_{t+1} = \frac{1}{N} \sum_{i=1}^N x_{t+1}^i = \frac{1}{N} \sum_{i=1}^N (x_t^i - \eta \nabla f(x_t^i, \xi_t^i)) = \hat{x}_t - \eta \frac{1}{N} \sum_{i=1}^N \nabla f(x_t^i, \xi_t^i).\nonumber
\end{equation}
We use $t_p$ to denote the last time to communicate, i.e., $t_p = \lfloor t / k \rfloor \cdot k$. Then, we get
\begin{equation}\label{eq_expansion}
  \hat{x}_t = \hat{x}_{t_p} - \frac{\eta}{N} \sum_{\tau = t_p}^{t-1} \sum_{i=1}^N \nabla f(x_{\tau}^i, \xi_{\tau}^i)~~~~~~{\rm and}~~~~~x_t^i = \hat{x}_{t_p} - \eta \sum_{\tau = t_p}^{t-1} \nabla f(x_{\tau}^i, \xi_{\tau}^i).
\end{equation}
As each client updates its model locally and communicates with others periodically, it is important to make sure that the divergence of local models is not very large.
We use Lemma~\ref{lemma4} to bound the difference between $\hat{x}_t$ and $x_t^i$ to guarantee this.
\begin{lemma}\label{lemma4}
Under Assumptions~1 and 2, for any $x \in R^d$, Algorithm~1 ensures that
\begin{equation} \label{lemma4:0}
  \frac{1}{N}  \sum_{i=1}^N  \sum_{t=0}^{T-1} \mathbb{E} \| \hat{x}_t - x_t^i \|^2 \leq \frac{k-1}{1 - 2k^2 \eta^2 L^2} \left( T \eta^2 \sigma^2 +  8 k \eta^2 L \sum_{t=0}^{T-1} 
  \mathbb{E} \mathcal{D}_{f}(\hat{x}_{\tau}, x)
   + 4T k \eta^2 \zeta_f^x \right).
\end{equation}
\end{lemma}
\begin{proof}
According to (\ref{eq_expansion}), we have
\begin{eqnarray}\label{lemma4:1}
   \| \hat{x}_t - x_t^i \|^2 &=&  \left \| \hat{x}_{t_p} - \frac{\eta}{N} \sum_{\tau = t_p}^{t-1} \sum_{j=1}^N \nabla f(x_{\tau}^j, \xi_{\tau}^j) - \left( \hat{x}_{t_p} - \eta \sum_{\tau = t_p}^{t-1} \nabla f(x_{\tau}^i, \xi_{\tau}^i) \right)   \right \|^2
   \nonumber\\
   &=& \eta^2 \left\| \sum_{\tau = t_p}^{t-1} \nabla f(x_{\tau}^i, \xi_{\tau}^i) - \frac{1}{N} \sum_{j=1}^N \sum_{\tau = t_p}^{t-1}   \nabla f(x_{\tau}^j, \xi_{\tau}^j) \right\|^2.
   \nonumber
\end{eqnarray}
Since $\frac{1}{N} \sum_{i=1}^N \left \| A_i - \frac{1}{N} \sum_{j=1}^N  A_j \right\|^2 = \frac{1}{N} \sum_{i=1}^N \left\| A_i \right\|^2 - \left\| \frac{1}{N} \sum_{i=1}^N  A_i \right\|^2$, we have
\begin{eqnarray} \label{lemma4:2}
  \frac{1}{N}  \sum_{i=1}^N  \mathbb{E} \left\| \hat{x}_t - x_t^i \right\|^2 &=& \eta^2  \left( \frac{1}{N} \sum_{i=1}^N  \mathbb{E} \left\| \sum_{\tau = t_p}^{t-1} \nabla f(x_{\tau}^i, \xi_{\tau}^i) \right\|^2 - \mathbb{E} \left\| \frac{1}{N} \sum_{i=1}^N \sum_{\tau = t_p}^{t-1} \nabla f(x_{\tau}^i, \xi_{\tau}^i) \right\|^2 \right)
  \nonumber\\
  &\leq& \frac{\eta^2}{N}  \sum_{i=1}^N \mathbb{E} \left\| \sum_{\tau = t_p}^{t-1} \nabla f(x_{\tau}^i, \xi_{\tau}^i) \right\|^2.
\end{eqnarray}
Next, we bound  
$\mathbb{E} \left\| \sum_{\tau = t_p}^{t-1} \nabla f(x_{\tau}^i, \xi_{\tau}^i) \right\|^2$:
\begin{eqnarray} \label{lemma4:3}
  \mathbb{E} \left\| \sum_{\tau = t_p}^{t-1} \nabla f(x_{\tau}^i, \xi_{\tau}^i) \right\|^2 &=& \mathbb{E} \left\| \sum_{\tau = t_p}^{t-1} \nabla f(x_{\tau}^i, \xi_{\tau}^i) - \sum_{\tau = t_p}^{t-1} \nabla f_i(x_{\tau}^i) + \sum_{\tau = t_p}^{t-1} \nabla f_i(x_{\tau}^i) \right\|^2
  \nonumber\\
  &\overset{(a)}{=}& \mathbb{E} \left\| \sum_{\tau = t_p}^{t-1} \nabla f(x_{\tau}^i, \xi_{\tau}^i) - \sum_{\tau = t_p}^{t-1} \nabla f_i(x_{\tau}^i) \right\|^2 + \mathbb{E} \left\|\sum_{\tau = t_p}^{t-1} \nabla f_i(x_{\tau}^i) \right\|^2
  \nonumber\\
  &\overset{(b)}{=}& \sum_{\tau = t_p}^{t-1} \mathbb{E} \left\| \nabla f(x_{\tau}^i, \xi_{\tau}^i) - \nabla f_i(x_{\tau}^i) \right\|^2 + \mathbb{E} \left\|\sum_{\tau = t_p}^{t-1} \nabla f_i(x_{\tau}^i) \right\|^2
  \nonumber\\
  &\overset{(c)}{\leq}& \sum_{\tau = t_p}^{t-1} \mathbb{E} \left\| \nabla f(x_{\tau}^i, \xi_{\tau}^i) - \nabla f_i(x_{\tau}^i) \right\|^2 + (t - t_p)\sum_{\tau = t_p}^{t-1} \mathbb{E} \left\| \nabla f_i(x_{\tau}^i) \right\|^2
  \nonumber\\
  &\overset{(d)}{\leq}& (t-t_p) \sigma^2 + (t - t_p)\sum_{\tau = t_p}^{t-1} \mathbb{E} \left\| \nabla f_i(x_{\tau}^i) \right\|^2,
\end{eqnarray}
where $(a)$ and $(b)$ hold because $\mathbb{E} \nabla f(x_{\tau}^i, \xi_{\tau}^i) = \nabla f_i(x_{\tau}^i)$ and $\xi_{\tau}^i$'s are independent; $(c)$ follows from Cauchy's inequality; $(d)$ is due to Assumption~2. We then bound $\mathbb{E} \left\|  \nabla f_i(x_{\tau}^i)  \right\|^2$:
\begin{eqnarray} \label{lemma4:4}
  \mathbb{E} \left\|  \nabla f_i(x_{\tau}^i)  \right\|^2 &=& \mathbb{E} \left\|  \nabla f_i(x_{\tau}^i) -  \nabla f_i(\hat{x}_{\tau}) + \nabla f_i(\hat{x}_{\tau}) \right\|^2
  \nonumber\\
  &\overset{(a)}{\leq}& 2 \mathbb{E} \| \nabla f_i(x_{\tau}^i) -  \nabla f_i(\hat{x}_{\tau}) \|^2 + 2 \mathbb{E} \| \nabla f_i(\hat{x}_{\tau}) \|^2
  \nonumber\\
  &\overset{(b)}{\leq}& 2 L^2 \mathbb{E} \| x_{\tau}^i -  \hat{x}_{\tau} \|^2 + 2 \mathbb{E} \| \nabla f_i(\hat{x}_{\tau}) - \nabla f_i(x) + \nabla f_i(x) \|^2
  \nonumber\\
  &\overset{(c)}{\leq}& 2 L^2 \mathbb{E} \| x_{\tau}^i -  \hat{x}_{\tau} \|^2 + 4 \mathbb{E} \| \nabla f_i(\hat{x}_{\tau}) - \nabla f_i(x) \|^2 + 4 \mathbb{E} \| \nabla f_i(x) \|^2
  \nonumber\\
  &\overset{(d)}{\leq}& 2 L^2 \mathbb{E} \| x_{\tau}^i -  \hat{x}_{\tau} \|^2 + 8L \mathbb{E} 
  \mathcal{D}_{f_i}(\hat{x}_{\tau}, x)
  + 4 \mathbb{E} \| \nabla f_i(x) \|^2,
\end{eqnarray}
where $(a)$ and $(c)$ come from $\| a+b \|^2 \leq 2\|a\|^2 + 2\| b\|^2$, $(b)$ holds because of Assumption~1, $(d)$ 
follows from Lemma~\ref{smooth_basic}.
Substituting (\ref{lemma4:4}) into (\ref{lemma4:3}) and based on $t - t_p \leq k - 1$, we have
\begin{eqnarray} \label{lemma4:5}
  &&\mathbb{E} \left\| \sum_{\tau = t_p}^{t-1} \nabla f(x_{\tau}^i, \xi_{\tau}^i) \right\|^2 
  \nonumber\\
  &\leq& (k-1) \sigma^2 + (k-1)\sum_{\tau = t_p}^{t-1} \left(2 L^2 \mathbb{E} \| x_{\tau}^i -  \hat{x}_{\tau} \|^2 + 8L \mathbb{E} 
  \mathcal{D}_{f_i}(\hat{x}_{\tau}, x) 
  + 4 \mathbb{E} \| \nabla f_i(x) \|^2 \right).
\end{eqnarray}
Substituting (\ref{lemma4:5}) into (\ref{lemma4:2}) and according to the definition of $\zeta_f^x$, we get
\begin{eqnarray} \label{lemma4:6}
  \frac{1}{N} \sum_{i=1}^N  \mathbb{E} \left\| \hat{x}_t - x_t^i \right\|^2
  &\leq& \eta^2 (k-1) \sigma^2 +  \frac{2 (k-1) \eta^2 L^2}{N} \sum_{i=1}^N  \sum_{\tau = t_p}^{t-1} \mathbb{E} \| x_\tau^i - \hat{x}_\tau \|^2 
  \nonumber\\
  &&+ 8 (k-1) \eta^2 L \sum_{\tau = t_p}^{t-1} \mathbb{E} 
  \mathcal{D}_{f}(\hat{x}_{\tau}, x)
  + 4 (k-1) \eta^2 \sum_{\tau = t_p}^{t-1} \zeta_f^x. 
  \nonumber
\end{eqnarray}
Summing up this inequality from $t=0$ to $T - 1$, we have
\begin{eqnarray} \label{lemma4:7}
  &&\frac{1}{N}  \sum_{i=1}^N  \sum_{t=0}^{T-1} \mathbb{E} \left\| \hat{x}_t - x_t^i \right\|^2 
  \nonumber\\
  &\leq& (k-1) \Bigg(T \eta^2\sigma^2 + \frac{2 \eta^2 L^2}{N} \sum_{i=1}^N  \sum_{t=0}^{T-1} \sum_{\tau = t_p}^{t-1}  \| x_\tau^i - \hat{x}_\tau \|^2 
  \nonumber\\
  &&+ 8 \eta^2 L \sum_{t=0}^{T-1} \sum_{\tau = t_p}^{t-1} 
  \mathbb{E} \mathcal{D}_{f}(\hat{x}_{\tau}, x)
  + 4 \eta^2  \sum_{t=0}^{T-1} \sum_{\tau = t_p}^{t-1} \zeta_f^x \Bigg)
  \nonumber\\
  &\leq& (k-1)\left(T \eta^2 \sigma^2 +  \frac{2 k \eta^2 L^2}{N}  \sum_{i=1}^N  \sum_{t=0}^{T-1} \| x_\tau^i - \hat{x}_\tau \|^2 
  + 8 k \eta^2 L \sum_{t=0}^{T-1} 
  \mathbb{E} \mathcal{D}_{f}(\hat{x}_{\tau}, x)
  + 4 T k \eta^2 \zeta_f^x \right),~~~~~~~
\end{eqnarray}
where the second inequality comes from a simple counting argument: $\sum_{t=0}^T \sum_{\tau=t_p}^{t-1} A_\tau \leq \sum_{t=0}^T \sum_{\tau=t-k}^{t-1} A_\tau \leq k \sum_{t=0}^T A_t, A_t \geq 0$. Rearranging (\ref{lemma4:7}), we get
\begin{equation} \label{lemma4:8}
   \frac{1}{N}  \sum_{i=1}^N  \sum_{t=0}^{T-1} \mathbb{E} \| \hat{x}_t - x_t^i \|^2 \leq \frac{k-1}{1 - 2k^2 \eta^2 L^2} \left( T \eta^2 \sigma^2 +  8 k \eta^2 L \sum_{t=0}^{T-1} 
  \mathbb{E} \mathcal{D}_{f}(\hat{x}_{\tau}, x)
   + 4T k \eta^2 \zeta_f^x \right).
   \nonumber
\end{equation}
\end{proof}

Below, we use Lemma~\ref{lemma5} to bound the average of stochastic gradients.
\begin{lemma}\label{lemma5}
  Under Assumptions~1 and 
  2,
  we have
  \begin{equation} \label{lemma5:0}
    \mathbb{E} \left \| \sum_{i=1}^N \frac{1}{N} \nabla f (x_t^i, \xi_t^i) \right\|^2 \leq \frac{\sigma^2}{N} + \frac{3L^2}{N} \sum_{i=1}^N \mathbb{E} \left \| x_t^i - \hat{x}_t \right\|^2 + 
    \frac{3}{2} \mathbb{E} \left \| \nabla f(\hat{x}_t)\right\|^2.
  \end{equation}
\end{lemma}
\begin{proof}
Since $\xi_t^i$'s are independent, we have
\begin{eqnarray} \label{lemma5:1}
  \mathbb{E} \left \|  \frac{1}{N} \sum_{i=1}^N  \nabla f (x_t^i, \xi_t^i) \right\|^2  &=& \mathbb{E} \left \| \frac{1}{N} \sum_{i=1}^N  \nabla f_i (x_t^i, \xi_t^i) -  \frac{1}{N} \sum_{i=1}^N  \nabla f_i (x_t^i) + \frac{1}{N} \sum_{i=1}^N \nabla f_i (x_t^i)  \right\|^2
  \nonumber\\
  &=&  \frac{1}{N^2} \mathbb{E} \left \| \sum_{i=1}^N  \nabla f (x_t^i, \xi_t^i) -   \sum_{i=1}^N  \nabla f_i (x_t^i) \right\|^2 +   \mathbb{E} \left \| \frac{1}{N} \sum_{i=1}^N  \nabla f_i (x_t^i)  \right\|^2
  \nonumber\\
  &=&  \frac{1}{N^2} \sum_{i=1}^N \mathbb{E} \left \| \nabla f (x_t^i, \xi_t^i) -  \nabla f_i (x_t^i) \right\|^2 +  \mathbb{E} \left \|  \frac{1}{N}\sum_{i=1}^N \nabla f_i (x_t^i)  \right\|^2
  \nonumber\\
  &\leq& \frac{\sigma^2}{N} +  \mathbb{E} \left \|\frac{1}{N} \sum_{i=1}^N \nabla f_i (x_t^i)  \right\|^2,
\end{eqnarray}
where the last inequality comes from Assumption~2. According to Young's Inequality and Cauchy's Inequality, we have
\begin{eqnarray} \label{lemma5:2}
  \mathbb{E} \left \| \frac{1}{N} \sum_{i=1}^N \nabla f_i (x_t^i)  \right\|^2
  &=&  \mathbb{E} \left \| \frac{1}{N} \sum_{i=1}^N \nabla f_i (x_t^i) - \frac{1}{N} \sum_{i=1}^N \nabla f_i (\hat{x}_t) + \frac{1}{N} \sum_{i=1}^N \nabla f_i (\hat{x}_t) \right\|^2
  \nonumber\\
  &\leq& 3 \mathbb{E} \left \| \frac{1}{N}  \sum_{i=1}^N \left( \nabla f_i (x_t^i) - \nabla f_i (\hat{x}_t) \right) \right\|^2 + \frac{3}{2} \mathbb{E} \left \| \frac{1}{N}  \sum_{i=1}^N \nabla f_i (\hat{x}_t) \right\|^2
  \nonumber\\
  &=&  \frac{3}{N^2} \mathbb{E} \left \| \sum_{i=1}^N \left( \nabla f_i (x_t^i) - \nabla f_i (\hat{x}_t) \right) \right\|^2 + \frac{3}{2} \mathbb{E} \left \| \nabla f(\hat{x}_t) \right\|^2
  \nonumber\\
  &\leq& \frac{3}{N} \sum_{i=1}^N \mathbb{E} \left \| \nabla f_i (x_t^i) - \nabla f_i (\hat{x}_t) \right\|^2 + \frac{3}{2} \mathbb{E} \left \| \nabla f(\hat{x}_t)\right\|^2
  \nonumber\\
  &\leq& \frac{3L^2}{N} \sum_{i=1}^N \mathbb{E} \left \| x_t^i - \hat{x}_t \right\|^2 + 
  \frac{3}{2} \mathbb{E} \left \| \nabla f(\hat{x}_t)\right\|^2,
\end{eqnarray}
where the last inequality holds since $f_i(x)$ is $L$-smooth. Substituting (\ref{lemma5:2}) into (\ref{lemma5:1}), we complete the proof.
\end{proof}

Next, we bounded $f(\hat{x}_t) - f(x)$ for any $x \in R^d$ with Lemma~\ref{lemma1}.
\begin{lemma}\label{lemma1}
  Suppose Assumptions~1 and 2 hold and $f(x)$ is convex. When Algorithm~1 runs with a fixed learning rate $\eta$, for any $x \in R^d$, we have
  \begin{eqnarray} \label{lemma1:0}
    && 2 \eta \sum_{t=0}^{T-1} \mathbb{E} \left( f(\hat{x}_t) - f(x)  \right) - \frac{3 \eta^2}{2} \sum_{t=0}^{T-1} \mathbb{E} \left \| \nabla f(\hat{x}_t)\right\|^2 
    \nonumber\\
    &&- \frac{(\eta L + 3 \eta^2 L^2)(k-1)}{1 - 2k^2 \eta^2 L^2} 8 k\eta^2 L \sum_{t=0}^{T-1} 
    \mathbb{E} \mathcal{D}_f(\hat{x}_t, x)
    \nonumber\\
    &\leq& \| \hat{x}_{0} - x^* \|^2 + \frac{T \eta^2 \sigma^2}{N}
    + \frac{(\eta L + 3 \eta^2 L^2)(k-1) }{1 - 2k^2 \eta^2 L^2} ( T \eta^2 \sigma^2 + 4T k \eta^2 \zeta_f^x).
  \end{eqnarray}
\end{lemma}
\begin{proof}
Based on the update rule of Algorithm~1, we obtain
\begin{eqnarray}\label{lemma1:1}
  \mathbb{E} \| \hat{x}_{t+1} - x \|^2 &=& \mathbb{E} \| \hat{x}_{t} - x \|^2 - 2 \eta \mathbb{E} \langle \hat{x}_t - x, \frac{1}{N} \sum_{i=1}^N \nabla f(x_t^i, \xi_t^i) \rangle + \eta^2 \mathbb{E} \left\| \frac{1}{N} \sum_{i=1}^N \nabla f(x_t^i, \xi_t^i) \right\|^2.
  \nonumber\\
  &=& \mathbb{E} \| \hat{x}_{t} - x \|^2 - 2 \eta  \mathbb{E} \langle \hat{x}_t - x, \frac{1}{N} \sum_{i=1}^N \nabla f_i(x_t^i) \rangle + \eta^2 \mathbb{E} \left \| \frac{1}{N} \sum_{i=1}^N \nabla f(x_t^i, \xi_t^i) \right\|^2.
\end{eqnarray}
Since $f_{i}(x)$ is convex and $L$-smooth, we have
\begin{eqnarray}\label{lemma1:2}
  &&- \langle \hat{x}_t - x, \frac{1}{N} \sum_{i=1}^N \nabla f_i(x_t^i) \rangle 
  \nonumber\\
  &=& \langle x - \hat{x}_t, \frac{1}{N} \sum_{i=1}^N \nabla f_i(x_t^i) \rangle
  \nonumber\\
  &=& \frac{1}{N} \sum_{i=1}^N \left( \langle x - x_t^i, \nabla f_i(x_t^i) \rangle + \langle x_t^i - \hat{x}_t, \nabla f_i(x_t^i) \rangle \right)
  \nonumber\\
  &\leq& \frac{1}{N} \sum_{i=1}^N \left( \left(f_i(x) - f_i(x_t^i) \right) + \left( f_i(x_t^i) - f_i(\hat{x}_t) + \frac{L}{2} \| x_t^i - \hat{x}_t \|^2 \right) \right)
  \nonumber\\
  &=& \frac{1}{N} \sum_{i=1}^N \left( f_i(x) - f_i(\hat{x}_t) + \frac{L}{2} \| x_t^i - \hat{x}_t \|^2  \right).
  \nonumber\\
  &=& f(x) - f(\hat{x}_t) + \frac{L}{2N} \sum_{i=1}^N \| x_t^i - \hat{x}_t \|^2
\end{eqnarray}
Substituting (\ref{lemma1:2}) into (\ref{lemma1:1}) yields
\begin{equation} \label{lemma1:3}
  \mathbb{E} \| \hat{x}_{t+1} - x \|^2 \leq \mathbb{E} \| \hat{x}_{t} - x \|^2 + 
  2 \eta \left( f(x) - f(\hat{x}_t) + \frac{L}{2N} \sum_{i=1}^N \| x_t^i - \hat{x}_t \|^2 \right)
  + \eta^2 \mathbb{E} \left \| \frac{1}{N} \sum_{i=1}^N \nabla f(x_t^i, \xi_t^i) \right \|^2.
\end{equation}
According to (\ref{lemma5:0}) in Lemma~\ref{lemma5}, we have
\begin{equation} \label{lemma1:4}
  \mathbb{E} \left \| \frac{1}{N} \sum_{i=1}^N  \nabla f (x_t^i, \xi_t^i) \right\|^2 \leq \frac{\sigma^2}{N} + \frac{3L^2}{N} \sum_{i=1}^N \mathbb{E} \left \| x_t^i - \hat{x}_t \right\|^2 + 
    \frac{3}{2} \mathbb{E} \left \| \nabla f(\hat{x}_t)\right\|^2
\end{equation}
Combining (\ref{lemma1:3}) and (\ref{lemma1:4}), we get
\begin{eqnarray}
  \mathbb{E} \| \hat{x}_{t+1} - x \|^2 &\leq& \mathbb{E} \| \hat{x}_{t} - x \|^2 + 2 \eta \left( f(x) - f(\hat{x}_t) + \frac{L}{2N} \sum_{i=1}^N \| x_t^i - \hat{x}_t \|^2 \right)
  \nonumber\\
  && + \eta^2 \left( \frac{\sigma^2}{N} + \frac{3L^2}{N} \sum_{i=1}^N \mathbb{E} \left \| x_t^i - \hat{x}_t \right\|^2 + \frac{3}{2} \mathbb{E} \left \| \nabla f(\hat{x}_t)\right\|^2 \right).
  \nonumber\\
  &=& \mathbb{E} \| \hat{x}_{t} - x \|^2 - 2 \eta \mathbb{E} \left( f(\hat{x}_t) - f(x)  \right) + \frac{3 \eta^2}{2} \mathbb{E} \left \| \nabla f(\hat{x}_t)\right\|^2
  \nonumber\\
  && + \frac{\eta L + 3 \eta^2 L^2}{N} \sum_{i=1}^N \| \hat{x}_t - x_t^i \|^2 + \frac{\eta^2 \sigma^2}{N}.
\end{eqnarray}
Summing up this inequality from $t=0$ to $T - 1$, we have
\begin{eqnarray} \label{lemma1:5}
\mathbb{E} \| \hat{x}_{T} - x \|^2 &\leq& \| \hat{x}_{0} - x \|^2 - 2 \eta \sum_{t=0}^{T-1} \mathbb{E} \left( f(\hat{x}_t) - f(x)  \right) + \frac{3 \eta^2}{2} \sum_{t=0}^{T-1} \mathbb{E} \left \| \nabla f(\hat{x}_t)\right\|^2
\nonumber\\
&& + \frac{\eta L + 3 \eta^2 L^2}{N} \sum_{i=1}^N \sum_{t=0}^{T-1} \| \hat{x}_t - x_t^i \|^2 + \frac{T \eta^2 \sigma^2}{N}.
\end{eqnarray}
Substituting (\ref{lemma4:0}) in Lemma~\ref{lemma4} into (\ref{lemma1:5}), it holds that
\begin{eqnarray} \label{lemma1:6}
  && 2 \eta \sum_{t=0}^{T-1} \mathbb{E} \left( f(\hat{x}_t) - f(x)  \right) 
  \nonumber\\
  &\leq& \| \hat{x}_{0} - x \|^2 - \mathbb{E} \| \hat{x}_{T} - x \|^2 + \frac{3 \eta^2}{2} \sum_{t=0}^{T-1} \mathbb{E} \left \| \nabla f(\hat{x}_t)\right\|^2
  \nonumber\\
  && + \frac{(\eta L + 3 \eta^2 L^2)(k-1)}{1 - 2k^2 \eta^2 L^2}\left( T \eta^2 \sigma^2 +  8 k\eta^2 L \sum_{t=0}^{T-1} 
  \mathbb{E} \mathcal{D}_f(\hat{x}_t, x)
  + 4T k \eta^2 \zeta_f^x \right) + \frac{T \eta^2 \sigma^2}{N}.
  \nonumber\\
  &\leq& \| \hat{x}_{0} - x \|^2 + \frac{3 \eta^2}{2} \sum_{t=0}^{T-1} \mathbb{E} \left \| \nabla f(\hat{x}_t)\right\|^2
  \nonumber\\
  && + \frac{(\eta L + 3 \eta^2 L^2)(k-1)}{1 - 2k^2 \eta^2 L^2}\left( T \eta^2 \sigma^2 +  8 k\eta^2 L \sum_{t=0}^{T-1} 
  \mathbb{E} \mathcal{D}_f(\hat{x}_t, x)
  + 4T k \eta^2 \zeta_f^x \right) + \frac{T \eta^2 \sigma^2}{N}.
\end{eqnarray}
Rearranging (\ref{lemma1:6}), we get
\begin{eqnarray} \label{lemma1:7}
  && 2 \eta \sum_{t=0}^{T-1} \mathbb{E} \left( f(\hat{x}_t) - f(x)  \right) - \frac{3 \eta^2}{2} \sum_{t=0}^{T-1} \mathbb{E} \left \| \nabla f(\hat{x}_t)\right\|^2 
  \nonumber\\
  &&- \frac{(\eta L + 3 \eta^2 L^2)(k-1)}{1 - 2k^2 \eta^2 L^2} 8 k\eta^2 L \sum_{t=0}^{T-1} 
  \mathbb{E} \mathcal{D}_f(\hat{x}_t, x)
  \nonumber\\
  &\leq& \| \hat{x}_{0} - x \|^2 
  + \frac{(\eta L + 3 \eta^2 L^2)(k-1) }{1 - 2k^2 \eta^2 L^2} ( T \eta^2 \sigma^2 + 4T k \eta^2 \zeta_f^x) + \frac{T \eta^2 \sigma^2}{N}.
\end{eqnarray}
\end{proof}

\subsection{Proof of Theorem~1}
\begin{proof}
Applying (\ref{lemma1:0}) in Lemma~\ref{lemma1} with $x = x^*$, it holds that
\begin{eqnarray}\label{theo1:1}
  && 2 \eta \sum_{t=0}^{T-1} \mathbb{E} \left( f(\hat{x}_t) - f(x^*)  \right) - \frac{3 \eta^2}{2} \sum_{t=0}^{T-1} \mathbb{E} \left \| \nabla f(\hat{x}_t)\right\|^2 
  \nonumber\\
  &&- \frac{(\eta L + 3 \eta^2 L^2)(k-1)}{1 - 2k^2 \eta^2 L^2} 8 k\eta^2 L \sum_{t=0}^{T-1} \mathbb{E} \mathcal{D}_f(\hat{x}_t, x^*)
  \nonumber\\
  &\leq& \| \hat{x}_{0} - x^* \|^2 + \frac{(\eta L + 3 \eta^2 L^2)(k-1) }{1 - 2k^2 \eta^2 L^2} ( T \eta^2 \sigma^2 + 4T k \eta^2 \zeta_f^*) + \frac{T \eta^2 \sigma^2}{N}.
\end{eqnarray}
As $f_i(x), i \in [N]$ are $L$-smooth, it is easy to verify that $f(x)$ is $L$-smooth. According to Lemma~\ref{smooth_basic}, we have
\begin{eqnarray}\label{theo1:2}
  \| \nabla f(\hat{x}_t) \|^2 &=& \| \nabla f(\hat{x}_t) - \nabla f(x^*) \|^2
  \nonumber\\
  &\leq& 2L \mathcal{D}_f (\hat{x}_t, x^*)
  \nonumber\\
  &=& 2L \left( f(\hat{x}_t) - f(x^*) \right).
\end{eqnarray}
Substituting (\ref{theo1:2}) into the left hand side of (\ref{theo1:1}) yields
\begin{eqnarray}\label{theo1:3}
  && \left(2 \eta - 3 \eta^2 L - \frac{(\eta L + 3 \eta^2 L^2)(k-1)}{1 - 2k^2 \eta^2 L^2} 8 k\eta^2 L \right) \sum_{t=0}^{T-1} \mathbb{E} \left( f(\hat{x}_t) - f(x^*)  \right)
  \nonumber\\
  &\leq& \| \hat{x}_{0} - x^* \|^2 + \frac{(\eta L + 3 \eta^2 L^2)(k-1) }{1 - 2k^2 \eta^2 L^2} ( T \eta^2 \sigma^2 + 4T k \eta^2 \zeta_f^*) + \frac{T \eta^2 \sigma^2}{N}.
\end{eqnarray}
Setting the learning rate $\eta$ so that $\eta \leq \frac{1}{6L}$ and $\eta k \leq \frac{1}{9L}$, we have
\begin{equation}\label{theo1:4}
  \frac{\eta L + 3 \eta^2 L^2}{1 - 2k^2 \eta^2 L^2} \leq \frac{ \eta L + \frac{\eta L}{2} }{ 1 - \frac{2}{81} } \leq \frac{7 \eta L}{4},
\end{equation}
and
\begin{eqnarray}\label{theo1:5}
  2 \eta - 3 \eta^2 L - \frac{(\eta L + 3 \eta^2 L^2) 8 (k-1)k \eta^2 L }{1 - 2k^2 \eta^2 L^2} &\geq& 2 \eta - 3 \eta^2 L - \frac{(\eta L + 3 \eta^2 L^2) 8 k^2 \eta^2 L }{1 - 2k^2 \eta^2 L^2}
  \nonumber\\
  &\geq& 2 \eta - \frac{\eta}{2} - \frac{(\eta + \frac{\eta}{2} ) 8 k^2 \eta^2 L^2}{ 1 - \frac{2}{81} } 
  \nonumber\\
  &\geq&  2 \eta - \frac{\eta}{2} - \frac{81}{79} \times \frac{3}{2} \times \frac{8}{81} \eta
  \nonumber\\
  &\geq& \frac{4}{3} \eta.
\end{eqnarray}
Substituting (\ref{theo1:4}) and (\ref{theo1:5}) into (\ref{theo1:3}), we get
\begin{equation}\label{theo1:6}
  \frac{4 \eta}{3} \sum_{t=0}^{T-1} \mathbb{E} \left( f(\hat{x}_t) - f(x^*)  \right) \leq \| \hat{x}_{0} - x^* \|^2 + \frac{7}{4}T \eta^3 L (k-1) (\sigma^2 + 4 k \zeta_f^*) + \frac{T \eta^2 \sigma^2}{N}.
  \nonumber
\end{equation}
Dividing by $\frac{4 \eta T}{3}$ on both sides of the above inequality yields
\begin{eqnarray}\label{theo1:7}
  \frac{1}{T} \sum_{t=0}^{T-1} \mathbb{E} \left( f(\hat{x}_t) - f(x^*)  \right) &\leq& \frac{3 \|\hat{x}_{0} - x^* \|^2}{4 \eta T} + \frac{21}{16} \eta^2 L (k-1) (\sigma^2 + 4 k \zeta_f^*) + \frac{3  \eta \sigma^2}{4N}.
  \nonumber\\
  &\leq&  \frac{3 \|\hat{x}_{0} - x^* \|^2}{4 \eta T} + \frac{3}{2} \eta^2 L (k-1) (\sigma^2 + 4 k \zeta_f^*) + \frac{3  \eta \sigma^2}{4N}.
  \nonumber
\end{eqnarray}
Recall that we let $\tilde{x} =  \hat{x}_t$ for randomly chosen $t$ from $\{ 0, 1, \cdots, T-1\}$. Taking the expectation with regard to $t$, 
we get
\begin{equation}\label{theo1:9}
  \mathbb{E} f(\tilde{x}) - f(x^*) \leq \frac{3 \|\hat{x}_{0} - x^* \|^2}{4 \eta T} + \frac{3}{2} \eta^2 L (k-1) (\sigma^2 + 4 k \zeta_f^*) + \frac{3  \eta \sigma^2}{4N}.
\end{equation}
Under the result of (\ref{theo1:9}), we set $k$ as
\begin{equation}\label{theo1:10}
  k = 
  \begin{cases}
    \min\{ \frac{1}{6 \eta L N}, \frac{1}{9 \eta L}\}~~~~~~~~~~~~~~~~~~~~~\zeta_f^* = 0,\\
    \min\{ \frac{\sigma}{\sqrt{6 \eta L N(\sigma^2 + 4\zeta_f)}}, \frac{1}{9 \eta L} \}~~~~~~else.
  \end{cases}
\end{equation}
For the IID case, i.e., $\zeta_f^* = 0$, based on the setting of $k$ in (\ref{theo1:10}), we have 
\begin{eqnarray}\label{theo1:11}
  \frac{3}{2} \eta^2 L (k-1) (\sigma^2 + 4 k \zeta_f^*) &\leq& \frac{3}{2} \eta^2 L k \sigma^2 
  \nonumber\\
  &\leq& \frac{3}{2} \eta^2 L \frac{1}{6 \eta L N} \sigma^2
  \nonumber\\
  &=&  \frac{\eta \sigma^2}{4 N}.
\end{eqnarray}
For the Non-IID case, we get
\begin{eqnarray}\label{theo1:12}
  \frac{3}{2} \eta^2 L (k-1) (\sigma^2 + 4 k \zeta_f^*) &\leq& \frac{3}{2} \eta^2 L k^2 (\sigma^2 + 4 \zeta_f^*)
  \nonumber\\
  &\leq& \frac{3}{2} \eta^2 L \frac{\sigma^2}{6 \eta L N(\sigma^2 + 4\zeta_f^*)} (\sigma^2 + 4 \zeta_f^*)
  \nonumber\\
  &=&  \frac{\eta \sigma^2}{4 N}.
\end{eqnarray}
Substituting (\ref{theo1:11}) and (\ref{theo1:12}) into (\ref{theo1:9}) yields
\begin{eqnarray}
  \mathbb{E} f(\tilde{x}) - f(x^*) \leq  \frac{3 \|\hat{x}_{0} - x^* \|^2}{4 \eta T} +  \frac{ \eta \sigma^2}{N},
\end{eqnarray}
which completes the proof.
\end{proof}

\section[D]{Proofs of Results for strongly convex problems}

\subsection*{Proof of Theorem~2}
\begin{proof}
Based on the parameter settings in Algorithm~2, we have
\begin{eqnarray}\label{theo2:1}
  \eta_s T_s = \frac{\eta_1}{2^{s-1}} \cdot 2^{s-1}T_1 = \eta_1 T_1 = \frac{6}{\mu} \label{setting_eta_T}
\end{eqnarray}
and
\begin{eqnarray}
  k_s &=& \begin{cases}
    (\sqrt{2})^{s-1} k_1\\
    2^{s-1} k_1
   \end{cases}
   \leq \begin{cases}
   (\sqrt{2})^{s-1}\min\{  \frac{\sigma}{\sqrt{6 \eta_1 L N(\sigma^2 + 4\zeta_f)}}, \frac{1}{9 \eta_1 L} \}\\
   2^{s-1}\min\{  \frac{1}{6 \eta_1 L N}, \frac{1}{9 \eta_1 L} \}
   \end{cases}
   \nonumber\\
   &=& \begin{cases}
   \min\{  \frac{\sigma}{\sqrt{6 \eta_s L N(\sigma^2 + 4\zeta_f)}}, \frac{1}{9 (\sqrt{2})^{s-1}\eta_s L} \}\\
   \min\{  \frac{1}{6 \eta_s L N}, \frac{1}{9 \eta_s L} \}
   \end{cases}
   \nonumber\\
   &\leq& \begin{cases}
    \min\{  \frac{\sigma}{\sqrt{6 \eta_s L N(\sigma^2 + 4\zeta_f)}}, \frac{1}{9 \eta_s L} \}, ~{\rm Non\text{-}IID~case},\\
    \min\{  \frac{1}{6 \eta_s L N}, \frac{1}{9 \eta_s L} \},~~~~~~~~~~~~~~~~~~{\rm IID~case}.
   \end{cases}\label{setting_k}
\end{eqnarray}
Thus, according to (\ref{setting_eta_T}), (\ref{setting_k}) and Theorem~1, we get
\begin{equation}\label{theo2:2}
  \mathbb{E} f(x_{s+1}) - f(x^*) \leq  \frac{3 \mathbb{E} \| x_{s} - x^* \|^2}{4 \eta_s T_s} +  \frac{ \eta_s \sigma^2}{N} = \frac{\mu \mathbb{E} \|x_{s} - x^* \|^2}{8} +  \frac{ \eta_1 \sigma^2}{2^{s-1} N}.
\end{equation}
Since the objective $f(x)$ is $\mu$-strongly convex, we have
\begin{equation}\label{theo2:3}
  \frac{\mu \mathbb{E} \|x_{s} - x^* \|^2}{8} \leq \frac{ \mathbb{E} f(x_s) - f(x^*)}{4}.
\end{equation}
Substituting (\ref{theo2:3}) into (\ref{theo2:2}) yields
\begin{equation}\label{theo2:4}
  \mathbb{E} f(x_{s+1}) - f(x^*) \leq \frac{ \mathbb{E} f(x_s) - f(x^*)}{4} + \frac{ \eta_1 \sigma^2}{2^{s-1} N}.
\end{equation} 
Subtracting $\frac{\eta_1 \sigma^2}{2^{s-2}}$ on both sides of (\ref{theo2:4}), we get
\begin{equation}\label{theo2:5}
  \mathbb{E} f(x_{s+1}) - f(x^*) - \frac{8 \eta_1 \sigma^2}{2^{s+1} N} \leq \frac{1}{4}(\mathbb{E} f(x_s) - f(x^*) - \frac{8 \eta_1 \sigma^2}{2^s N}).
  \nonumber\\
\end{equation} 
Based on the property of geometric progression, we have
\begin{equation}\label{theo2:6}
  \mathbb{E} f(x_S) - f(x^*) - \frac{8 \eta_1 \sigma^2}{2^S N} \leq \frac{1}{4^{S-1}}(\mathbb{E} f(x_1) - f(x^*) - \frac{4 \eta_1 \sigma^2}{N}).
\end{equation}
Setting $S \geq \log(\frac{N(f(x_1) - f(x^*))}{\eta_1 \sigma^2}) + 2$ gives
\begin{equation}\label{theo2:7}
  f(x_1) - f(x^*) \leq \frac{2^{S-2} \eta_1 \sigma^2}{N}.
\end{equation}
By substituting (\ref{theo2:7}) into (\ref{theo2:6}) and rearranging the result further, we obtain
\begin{eqnarray}\label{theo2:8}
  \mathbb{E} f(x_S) - f(x^*) &\leq& \frac{8 \eta_1 \sigma^2}{2^S N} + \frac{1}{4^{S-1}}(\mathbb{E} f(x_1) - f(x^*) - \frac{4 \eta_1 \sigma^2}{N}) 
  \nonumber\\
  &\leq& \frac{8 \eta_1 \sigma^2}{2^S N} + \frac{\mathbb{E} f(x_1) - f(x^*)}{4^{S-1}}
  \nonumber\\
  &\leq& \frac{8 \eta_1 \sigma^2}{2^S N} + \frac{\eta_1 \sigma^2}{2^{S} N} 
  \nonumber\\
  &=& \frac{9 \eta_1 \sigma^2}{2^{S} N}.
\end{eqnarray}
Since $T_s = 2^{s-1} T_1$, we have
\begin{eqnarray}\label{theo2:9}
  T &=& T_1 + T_2 + \cdots + T_S 
  \nonumber\\
  &=& T_1 (1 + 2 + \cdots + 2^{S-1}) 
  \nonumber\\
  &=& T_1 (2^S - 1).\nonumber
\end{eqnarray}
Thus, it holds that
\begin{equation}\label{theo2:10}
  S = \log{(\frac{T}{T_1}+1)}.\nonumber
\end{equation}
Replacing $S$ with $\log(\frac{T}{T_1} + 1)$ in (\ref{theo2:8}) and combining $\eta_1 T_1 = \frac{6}{\mu}$, we have
\begin{eqnarray}\label{theo2:11}
  \mathbb{E} f(x_S) - f(x^*) &\leq& \frac{9 \eta_1 \sigma^2}{(\frac{T}{T_1} + 1)N} 
  \nonumber\\
  &=& \frac{9 \eta_1 T_1 \sigma^2}{(T+ T_1)N} 
  \nonumber\\
  &=& \frac{54 \sigma^2}{\mu (T+ T_1)N}
  \nonumber\\
  &=& O\left(\frac{1}{NT}\right).
  \nonumber
\end{eqnarray} 
\end{proof}

\section[E]{Proofs of Results for Non-Convex Problems}
\subsection{Proof of result for $\text{STL-SGD}^{nc}$ with \textbf{Option~1}}
We will first analyse the convergence of Local-SGD for a single stage in Lemma~\ref{lemma_section4_2}. Then we extend the result to $S$ stages in Theorem~3.
\begin{lemma}\label{lemma_section4_2} 
  Suppose Assumptions~1, 2 and 3 hold. Let $\gamma^{-1} = 2\rho$, $\eta_s \leq \frac{1}{12L_{\gamma}}$ and $k_s \eta_s \leq \frac{1}{9L_{\gamma}}$, where $L_{\gamma} = L + \gamma^{-1}$. We have the following result for stage $s$ of Algorithm~3 with \textbf{Option 1}:
  \begin{eqnarray}
  &&\mathbb{E} f(x_{s+1}) - f(x^*) 
  \nonumber\\ 
  &\leq& \left( \frac{3}{4 \eta_s T_s} +  \frac{1127 \rho}{632} \right)  \|x_s - x^* \|^2 + \frac{3\eta_s \sigma^2}{4N} +
  \frac{3}{2} \eta_s^2 L_{\gamma} (k_s-1) (\sigma^2 + 4k_s \zeta_f^*).
  \end{eqnarray}
\end{lemma}
\begin{proof}
We let the objectives in all stages be convex by setting $\gamma^{-1} > \rho$, where $\rho$ is the weakly convex parameter in Assumption~3. Recall that $f(x)$ is $L$-smooth. Denoting $L_{\gamma} = L + \frac{1}{\gamma}$, we have
\begin{eqnarray}\label{lemma2:1}
  \| \nabla f_{x_s}^{\gamma}(x) - \nabla f_{x_s}^{\gamma}(y) \| &=& \left \| \nabla f(x) - \nabla f(y) + \frac{1}{\gamma} ( x - y ) \right\|
  \nonumber\\
  &\leq& \| \nabla f(x) - \nabla f(y) \| + \frac{1}{\gamma} \left\|  ( x - y ) \right \|
  \nonumber\\
  &\leq& \left( L + \frac{1}{\gamma} \right) \| x - y \|
  \nonumber\\
  &=& L_\gamma \| x - y \|,
\end{eqnarray}
where the first inequality comes from the Triangle Inequality. Thus, $f_{x_s}^\gamma(x)$ is $L_{\gamma}$-smooth.
Based on Assumption~2, we further have
\begin{eqnarray}\label{lemma2:2}
  \mathbb{E}_{\xi \sim \mathcal{D}_i} \| \nabla f_{x_s}^\gamma(x, \xi) - \nabla f_{x_s,i}^\gamma(x) \|^2 = \mathbb{E}_{\xi \sim \mathcal{D}_i} \| \nabla f(x, \xi) - \nabla f_i(x) \|^2 \leq \sigma^2.
\end{eqnarray}
As we set $\gamma^{-1} > \rho$, $f_{x_s}^{\gamma}$ is ($\gamma^{-1} - \rho$)-strongly convex, thus we have
\begin{eqnarray}\label{lemma2:3}
  &&- \langle \hat{x}_t - x, \frac{1}{N} \sum_{i=1}^N \nabla f_{x_s,i}^{\gamma}(x_t^i) \rangle 
  \nonumber\\
  &=& \langle x - \hat{x}_t, \frac{1}{N} \sum_{i=1}^N \nabla f_{x_s,i}^{\gamma}(x_t^i) \rangle
  \nonumber\\
  &=& \frac{1}{N} \sum_{i=1}^N \left( \langle x - x_t^i, \nabla f_{x_s,i}^{\gamma}(x_t^i) \rangle + \langle x_t^i - \hat{x}_t, \nabla f_{x_s,i}^{\gamma}(x_t^i) \rangle \right)
  \nonumber\\
  &\leq& \frac{1}{N} \sum_{i=1}^N \Bigg( \left(f_{x_s,i}^{\gamma}(x) - f_{x_s,i}^{\gamma}(x_t^i) - \frac{\gamma^{-1} - \rho}{2} \| x_t^i - x \|^2 \right) 
  \nonumber\\
  &&+ \left( f_{x_s,i}^{\gamma}(x_t^i) - f_{x_s,i}^{\gamma}(\hat{x}_t) + \frac{L}{2} \| x_t^i - \hat{x}_t \|^2 \right) \Bigg)
  \nonumber\\
  &=& \frac{1}{N} \sum_{i=1}^N \left( f_{x_s,i}^{\gamma}(x) - f_{x_s,i}^{\gamma}(\hat{x}_t) + \frac{L}{2} \| x_t^i - \hat{x}_t \|^2 - \frac{\gamma^{-1} - \rho}{2} \| x_t^i - x \|^2 \right).
  \nonumber\\
  &\leq& f_{x_s}^{\gamma}(x) - f_{x_s}^{\gamma}(\hat{x}_t) + \frac{L}{2N} \sum_{i=1}^N \| x_t^i - \hat{x}_t \|^2 - \frac{\gamma^{-1} - \rho}{2} \| \hat{x}_t - x \|^2,
\end{eqnarray}
where the last inequality holds because the function $g(x)=\| x \|^2$ is convex. Respectively replacing (\ref{lemma1:2}) with (\ref{lemma2:3}), $L$ with $L_\gamma$ and $x$ with $x^*$, going through the proof process in Lemma~\ref{lemma1} again, we get 
\begin{eqnarray} \label{lemma2:4}
  && 2 \eta_s \sum_{t=0}^{T_s-1} \mathbb{E} \left( f_{x_s}^\gamma(\hat{x}_t) - f_{x_s}^\gamma(x^*)  \right) - \frac{3 \eta_s^2}{2} \sum_{t=0}^{T_s-1} \mathbb{E} \left \| \nabla f_{x_s}^\gamma(\hat{x}_t)\right\|^2 
  \nonumber\\
  &&- \frac{(\eta_s L_{\gamma} + 3 \eta_s^2 L_{\gamma}^2)(k_s-1)}{1 - 2k_s^2 \eta_s^2 L_{\gamma}^2} 8 k_s\eta_s^2 L_{\gamma} \sum_{t=0}^{T_s-1}
  \mathbb{E} \mathcal{D}_{f_{x_s}^\gamma}(\hat{x}_t, x^*)
  \nonumber\\
  &\leq& \| \hat{x}_{0} - x^* \|^2 - \eta_s(\gamma^{-1} - \rho) \sum_{t=0}^{T_s-1} \mathbb{e} \| \hat{x}_t - x^* \|^2
  \nonumber\\
  &&+ \frac{(\eta_s L_{\gamma} + 3 \eta_s^2 L_{\gamma}^2)(k_s-1) }{1 - 2k_s^2 \eta_s^2 L_{\gamma}^2} ( T_s \eta_s^2 \sigma^2 + 4T_s k_s \eta_s^2 \zeta_{f_{x_s}^\gamma}^*) + \frac{T_s \eta_s^2 \sigma^2}{N},
\end{eqnarray}
where $\mathcal{D}_{f_{x_s}^\gamma}(\hat{x}_t, x^*) = f_{x_s}^\gamma(\hat{x}_t) - f_{x_s}^\gamma(x^*) - \langle \nabla f_{x_s}^\gamma(x^*), \hat{x}_t - x^* \rangle$ and $\zeta_{f_{x_s}^\gamma}^* = \frac{1}{N} \sum_{i=1}^N \| \nabla f_i(x^*) + \frac{x^* - x_s}{\gamma} \|^2$. 
We bound $\| \nabla f_{x_s}^\gamma(\hat{x}_t)\|^2$ as
\begin{eqnarray}\label{lemma2:5}
  \| \nabla f_{x_s}^\gamma(\hat{x}_t)\|^2 &=& \| \nabla f_{x_s}^\gamma(\hat{x}_t) - \nabla f_{x_s}^\gamma(x^*) + \nabla f_{x_s}^\gamma(x^*)\|^2
  \nonumber\\
  &\leq& 2 \| \nabla f_{x_s}^\gamma(\hat{x}_t) - \nabla f_{x_s}^\gamma(x^*) \|^2 + 2\| \nabla f_{x_s}^\gamma(x^*) \|^2
  \nonumber\\
  &\leq& 4L_{\gamma} \mathcal{D}_{f_{x_s}^\gamma}(\hat{x}_t, x^*) + \frac{2}{\gamma^2} \| x^* - x_s \|^2,
\end{eqnarray}
where the last inequality comes from Lemma~\ref{smooth_basic}.
As $\frac{1}{N} \sum_{i=1}^N \nabla f_{i}(x^*) = \nabla f(x^*) = 0$, we have
\begin{eqnarray}\label{lemma2:6}
  \zeta_{f_{x_s}^\gamma}^* &=& \frac{1}{N} \sum_{i=1}^N \| \nabla f_i(x^*) + \frac{x^* - x_s}{\gamma} \|^2
  \nonumber\\
  &=& \frac{1}{N} \sum_{i=1}^N \| \nabla f_i(x^*)\|^2 + \|\frac{x^* - x_s}{\gamma} \|^2
  \nonumber\\
  &=& \zeta_f^* + \frac{1}{\gamma^2} \|x^* - x_s\|^2
\end{eqnarray}
and
\begin{eqnarray}\label{lemma2:7}
  \mathcal{D}_{f_{x_s}^\gamma}(\hat{x}_t, x^*) &=& f_{x_s}^{\gamma}(\hat{x}_t) - f_{x_s}^{\gamma}(x^*) + \frac{1}{\gamma} \langle x^* - x_s, x^* - \hat{x}_t \rangle
  \nonumber\\
  &\overset{(a)}{=}& f_{x_s}^{\gamma}(\hat{x}_t) - f_{x_s}^{\gamma}(x^*) + \frac{1}{2 \gamma} \left( \|x^* - x_s \|^2 - \| x_s - \hat{x}_t \|^2 + \| x^* - \hat{x}_t \|^2 \right),
\end{eqnarray}
where ($a$) is based on the fact that $\langle x-y, x-z \rangle = \frac{1}{2} \| x-y \|^2 - \frac{1}{2} \| y-z \|^2 + \frac{1}{2} \| x - z \|^2$. 
Substituting (\ref{lemma2:5}), (\ref{lemma2:6}), (\ref{lemma2:7}) into (\ref{lemma2:4}) and taking the expectation regarding $t$, we get
\begin{eqnarray}\label{lemma2:8}
  && T_s \left( 2\eta_s - 6 \eta_s^2 L_{\gamma} - 8A_{\gamma} k_s\eta_s^2 L_{\gamma} \right)  \left( f_{x_s}^\gamma(x_{s+1}) - f_{x_s}^\gamma(x^*) \right) 
  \nonumber\\
  &&- \left(\frac{3 \eta_s^2 T_s}{\gamma^2} + \frac{3\eta_s^2 L_{\gamma} T_s}{\gamma} + \frac{4A_{\gamma} k_s \eta_s^2 L_{\gamma} T_s}{\gamma} \right) \| x^* - x_s \|^2
  \nonumber\\
  &\leq&  (1 + \frac{4A_{\gamma} k_s \eta_s^2 T_s}{\gamma^2} )\| x_s - x^* \|^2 + \left(\frac{4A_{\gamma} k_s \eta_s^2 L_{\gamma} }{\gamma} + \frac{3 \eta_s^2 L_{\gamma}}{\gamma} - \eta_s(\gamma^{-1} - \rho) \right) \sum_{t=0}^{T_s-1} \| \hat{x}_t - x^* \|^2
  \nonumber\\
  &&+ A_{\gamma} T_s \eta_s^2( \sigma^2 + 4 k_s \zeta_f^*) + \frac{T_s \eta_s^2 \sigma^2}{N},
\end{eqnarray}
where $A_{\gamma} = \frac{(\eta_s L_{\gamma} + 3 \eta_s^2 L_{\gamma}^2)(k_s-1)}{1 - 2k_s^2 \eta_s^2 L_{\gamma}^2}$. 
Setting $\gamma = \frac{1}{2\rho}$, $\eta_s \leq \frac{1}{12L_{\gamma}}$ and $\eta_s k_s \leq \frac{1}{9L_{\gamma}}$, we have
\begin{eqnarray}\label{lemma2:9}
  A_\gamma k_s \eta_s^2 = \frac{(\eta_s L_{\gamma} + 3 \eta_s^2 L_{\gamma}^2)(k_s-1)}{1 - 2k_s^2 \eta_s^2 L_{\gamma}^2} k_s \eta_s^2 \leq \frac{(\eta_s + \frac{\eta_s}{4})k_s^2\eta_s^2L_{\gamma}^2 }{(1 - \frac{2}{81}) L_{\gamma}} \leq \frac{\frac{5\eta_s}{4}}{\frac{79}{81}L_{\gamma}} \frac{1}{81} = \frac{5\eta_s}{316L_{\gamma}},
\end{eqnarray}
\begin{eqnarray}\label{lemma2:10}
  2 \eta_s - 6 \eta_s^2 L_\gamma - 8 A_\gamma k_s \eta_s^2 L_\gamma \geq 2 \eta_s - \frac{\eta_s}{2} - \frac{10}{79} \eta_s \geq \frac{4}{3}\eta_s
\end{eqnarray}
and
\begin{eqnarray}\label{lemma2:11}
  \frac{4A_\gamma k_s \eta_s^2 L_\gamma}{\gamma} + \frac{3 \eta_s^2 L_\gamma}{\gamma} - \eta_s(\gamma^{-1} - \rho) \leq \frac{10\eta_s \rho}{79} + \frac{\eta_s \rho}{2} - \eta_s \rho \leq 0.
\end{eqnarray}
Substituting (\ref{lemma2:9}), (\ref{lemma2:10}) and (\ref{lemma2:11}) into (\ref{lemma2:8}) yields
\begin{eqnarray}\label{lemma2:12}
  &&\frac{4\eta_s T_s}{3} \left( f_{x_s}^\gamma(x_{s+1}) - f_{x_s}^\gamma(x^*)  \right) 
  \nonumber\\
  &\leq& \left( 1 + \frac{20T_s\eta_s \rho^2}{79 L_{\gamma}} + 12 \eta_s^2 T_s \rho^2 + 6 \eta_s^2 L_\gamma T_s \rho + \frac{10 T_s \eta_s \rho}{79} \right)  \|x_s - x^* \|^2 
  \nonumber\\
  &&+ \frac{3}{2} T_s \eta_s^3 L_{\gamma} (k_s-1) (\sigma^2 + 4k_s \zeta_f^*) + \frac{T_s \eta_s^2 \sigma^2}{N}.
\end{eqnarray}
By the definition of $f_{x_s}^\gamma(x)$ and $\gamma^{-1} = 2\rho$, we have
\begin{eqnarray}\label{lemma2:13}
  f_{x_s}^\gamma(x_{s+1}) - f_{x_s}^\gamma(x^*) &=& f(x_{s+1}) - f(x^*) + \rho \| x_{s+1} - x_{s} \|^2 - \rho \| x^* - x_s \|^2
  \nonumber\\
  &\geq& f(x_{s+1}) - f(x^*)  - \rho \| x^* - x_s \|^2.
\end{eqnarray}
Substituting (\ref{lemma2:13}) into (\ref{lemma2:12}) and rearranging the result further, we get
\begin{eqnarray}\label{lemma2:14}
  &&\frac{4\eta_s T_s}{3} \left( f(x_{s+1}) - f(x^*)  \right) 
  \nonumber\\
  &\leq& \left( 1 + \frac{20T_s\eta_s \rho^2}{79 L_{\gamma}} + 12 \eta_s^2 T_s \rho^2 + 6 \eta_s^2 L_\gamma T_s \rho + \frac{10 T_s \eta_s \rho}{79} + \frac{4 \eta_s T_s \rho}{3}\right)  \|x_s - x^* \|^2 
  \nonumber\\
  &&+ \frac{3}{2} T_s \eta_s^3 L_{\gamma} (k_s-1) (\sigma^2 + 4k_s \zeta_f^*) + \frac{T_s \eta_s^2 \sigma^2}{N}.
  \nonumber
\end{eqnarray}
Dividing by $\frac{4\eta_s T_s}{3}$ on both sides of the above inequality yields
\begin{eqnarray}\label{lemma2:15}
  f(x_{s+1}) - f(x^*)  &\leq& \left( \frac{3}{4 \eta_s T_s} + \frac{15 \rho^2}{79 L_{\gamma}} + 9 \eta_s \rho^2 + \frac{9 \eta_s L_\gamma \rho}{2} + \frac{15 \rho}{158} + \rho \right)  \|x_s - x^* \|^2 
  \nonumber\\
  &&+ \frac{3}{2} \eta_s^2 L_{\gamma} (k_s-1) (\sigma^2 + 4k_s \zeta_f^*) + \frac{3\eta_s \sigma^2}{4N}.
  \nonumber
\end{eqnarray}
As $L \geq \rho$, we have $L_{\gamma} = L + \frac{1}{\gamma} \geq 3 \rho$, $\eta_s \leq \frac{1}{12L_\gamma} \leq \frac{1}{36\rho} $ and
\begin{eqnarray}\label{lemma2:16}
  f(x_{s+1}) - f(x^*)  \leq \left( \frac{3}{4 \eta_s T_s} +  \frac{1127 \rho}{632} \right)  \|x_s - x^* \|^2 + \frac{3}{2} \eta_s^2 L_{\gamma} (k_s-1) (\sigma^2 + 4k_s \zeta_f^*) + \frac{3\eta_s \sigma^2}{4N}.
  \nonumber
\end{eqnarray}
\end{proof}

\paragraph{Proof of Theorem~3}
\begin{proof}
Since $f(x)$ satisfies the PL condition with parameter $\mu$, we have
\begin{eqnarray}\label{theo4:1}
  \frac{\mu}{2}\| x - x^* \|^2 \leq f(x) - f(x^*).
\end{eqnarray}
Combining (\ref{theo4:1}) with the result of Lemma~\ref{lemma_section4_2}, we have
\begin{eqnarray}\label{theo4:2}
  &&f(x_{s+1}) - f(x^*)  
  \nonumber\\
  &\leq& \left( \frac{3}{4 \eta_s T_s} +  \frac{1127 \rho}{632} \right)  \|x_s - x^* \|^2 + \frac{3}{2} \eta_s^2 L_{\gamma} (k_s-1) (\sigma^2 + 4k_s \zeta_f^*) + \frac{3\eta_s \sigma^2}{4N}
  \nonumber\\
  &\leq& \left( \frac{3}{4 \eta_s T_s} +  \frac{1127 \rho}{632} \right) \frac{2}{\mu} \left( f(x_s) - f(x^*) \right) + \frac{3}{2} \eta_s^2 L_{\gamma} (k_s-1) (\sigma^2 + 4k_s \zeta_f^*) + \frac{3\eta_s \sigma^2}{4N}.
\end{eqnarray}
According to the parameter settings in \textbf{Option~1} of Algorithm~3, we have
\begin{eqnarray}\label{theo4:3}
  \eta_s T_s = \frac{\eta_1}{2^{s-1}} \cdot 2^{s-1}T_1 = \eta_1 T_1 = \frac{6}{\rho} \label{setting_eta_T_4}
\end{eqnarray}  
and
\begin{eqnarray}
  k_s &=& \begin{cases}
    (\sqrt{2})^{s-1} k_1\\
    2^{s-1} k_1
   \end{cases}
   \leq \begin{cases}
   (\sqrt{2})^{s-1}\min\{  \frac{\sigma}{\sqrt{6 \eta_1 L_{\gamma} N(\sigma^2 + 4\zeta_f)}}, \frac{1}{9 \eta_1 L_{\gamma}} \}\\
   2^{s-1}\min\{  \frac{1}{6 \eta_1 L_{\gamma} N}, \frac{1}{9 \eta_1 L_{\gamma}} \}
   \end{cases}
   \nonumber\\
   &=& \begin{cases}
   \min\{  \frac{\sigma}{\sqrt{6 \eta_s L_{\gamma} N(\sigma^2 + 4\zeta_f)}}, \frac{1}{9 (\sqrt{2})^{s-1}\eta_s L_{\gamma}} \}\\
   \min\{  \frac{1}{6 \eta_s L_{\gamma} N}, \frac{1}{9 \eta_s L_{\gamma}} \}
   \end{cases}
   \nonumber\\
   &\leq& \begin{cases}
    \min\{  \frac{\sigma}{\sqrt{6 \eta_s L_{\gamma} N(\sigma^2 + 4\zeta_f)}}, \frac{1}{9 \eta_s L_{\gamma}} \}, ~{\rm Non\text{-}IID~case},\\
    \min\{  \frac{1}{6 \eta_s L_{\gamma} N}, \frac{1}{9 \eta_s L_{\gamma}} \},~~~~~~~~~~~~~~~~~~{\rm IID~case}.
   \end{cases}\label{setting_k_4}
\end{eqnarray}
Similar to the proof of (\ref{theo1:11}) and (\ref{theo1:12}), we have
\begin{eqnarray}\label{theo4:4}
  \frac{3}{2} \eta_s^2 L_{\gamma} (k_s-1) (\sigma^2 + 4k_s \zeta_f^*) \leq \frac{\eta_s \sigma^2}{4N}.
\end{eqnarray}
Substituting (\ref{setting_eta_T_4}) and (\ref{theo4:4}) into (\ref{theo4:2}), according to $\mu \geq 16 \rho$, we have
\begin{eqnarray}\label{theo4:5}
  f(x_{s+1}) - f(x^*) &\leq& \left( \frac{3}{4 \eta_s T_s} +  \frac{1127 \rho}{632} \right) \frac{2}{\mu} \left( f(x_s) - f(x^*) \right) + \frac{\eta_s \sigma^2}{N}.
  \nonumber\\
  &=& \left( \frac{\rho}{8} + \frac{1127\rho}{632} \right) \frac{2}{\mu} \left( f(x_s) - f(x^*) \right) + \frac{\eta_1 \sigma^2}{2^{s-1} N}
  \nonumber\\
  &\leq& \frac{1}{4} \left( f(x_s) - f(x^*) \right) + \frac{\eta_1 \sigma^2}{2^{s-1} N}.
\end{eqnarray}
Note that the formula of (\ref{theo4:5}) is the same as (\ref{theo2:4}). Thus, the rest of the proof is a duplicate to that of Theorem~2.
\end{proof}

\subsection{Proof for result of $\text{STL-SGD}^{nc}$ with \textbf{Option~2}}
\paragraph{Proof of Theorem~4}
\begin{proof}
  For convenience of analysis, we let $x_{s}^*$ denote the optimal solution of the objective used in the $s$-th stage $f_{x_s}^{\gamma}(x)$. According to (\ref{lemma2:1}) and (\ref{lemma2:2}), we have that $f_{x_s}^{\gamma}$ is $L_\gamma$-smooth and the variance of its stochastic gradients is bounded by $\sigma^2$. We set $\eta_1 \leq \frac{1}{6L_\gamma}$, $k_1 =  \min\{ \frac{1}{6 \eta_1 L_{\gamma} N}, \frac{1}{9 \eta_1 L_{\gamma}}\}$ when $\zeta_f^* = 0$ and $k_1 = \min\{ \frac{\sigma}{\sqrt{6 \eta_1 L_{\gamma} N(\sigma^2 + 4\zeta_f^*)}}, \frac{1}{9 \eta_1 L_{\gamma}} \} $ when $\zeta_f^* \neq 0$. 
  As $\eta_s = \eta_1 / s$ and $k_s = \begin{cases}
    s k_1,~~~~~~\textrm{IID case}\\
    \sqrt{s} k_1,~~~\textrm{else}
  \end{cases}$
  , we have 
  \begin{eqnarray}
    \eta_s \leq \frac{1}{6L_\gamma}
  \end{eqnarray}
  and
  \begin{eqnarray}\label{theo5:1}
    k_s \leq \begin{cases}
      \min\{ \frac{1}{6 \eta_s L_{\gamma} N}, \frac{1}{9 \eta_s L_{\gamma}}\},~~~~~~~~~~~~~~~~~~~~\textrm{IID case},\\
      \min\{ \frac{\sigma}{\sqrt{6 \eta_s L_{\gamma} N(\sigma^2 + 4\zeta_f^*)}}, \frac{1}{9 \eta_s L_{\gamma}} \},~~~\textrm{else}.
    \end{cases}
  \end{eqnarray}
  By setting $\gamma^{-1} > \rho$, we can ensure that $f_{x_s}^{\gamma}$ is strongly convex. Based on these settings, we apply Theorem~1 in each call of Local-SGD in $\text{STL-SGD}^{nc}$:
  \begin{eqnarray}\label{theo5:2}
    f_{x_s}^{\gamma}(x_{s+1}) - f_{x_s}^{\gamma}(x_s^*) \leq \frac{3\| x_s - x_s^* \|^2}{4\eta_s T_s} + \frac{\eta_s \sigma^2}{N}.
  \end{eqnarray}
  Under the definition $f_{x_s}^{\gamma}(x_{s+1}) = f(x_{s+1}) + \frac{1}{2\gamma} \| x_{s+1} - x_s \|^2$, and the strong convexity $f_{x_s}^{\gamma}(x_s) - f_{x_s}^{\gamma}(x_s^*) \geq \frac{\gamma^{-1}-\rho}{2} \| x_s - x_s^* \|^2$, we have
  \begin{eqnarray}\label{theo5:3}
    f(x_{s+1}) + \frac{1}{2\gamma} \| x_{s+1} - x_s \|^2 + \frac{\gamma^{-1}-\rho}{2} \| x_s - x_s^* \|^2 - f(x_s) \leq \frac{3\| x_s - x_s^* \|^2}{4\eta_s T_s} + \frac{\eta_s \sigma^2}{N}.
  \end{eqnarray}
  Setting $\gamma^{-1} = 2 \rho$ and rearranging (\ref{theo5:3}) yields
  \begin{eqnarray}\label{theo5:4}
    \rho \| x_{s+1} - x_s \|^2 + \frac{\rho}{2} \| x_s - x_s^* \|^2 \leq f(x_s) - f(x_{s+1}) + \frac{3\| x_s - x_s^* \|^2}{4\eta_s T_s} + \frac{\eta_s \sigma^2}{N}.
  \end{eqnarray}
  As $\eta_s = \eta_1 / s$, $T_s = s T_1$ and $\eta_1 T_1 = \frac{3}{\rho}$, we have
  \begin{eqnarray}\label{theo5:5}
    \rho \| x_{s+1} - x_s \|^2 + \frac{\rho}{4} \| x_s - x_s^* \|^2 \leq f(x_s) - f(x_{s+1})  + \frac{\eta_1 \sigma^2}{s N}.
  \end{eqnarray}
According to the $L_\gamma$-smoothness of $f_{x_s}^{\gamma}(x)$, we have
\begin{eqnarray}\label{theo5:6}
  \| \nabla f(x_s) \|^2 = \| \nabla f_{x_s}^{\gamma}(x_s) \|^2
  = \| \nabla f_{x_s}^{\gamma}(x_s) -  \nabla f_{x_s}^{\gamma}(x_s^*)\|^2
  \leq L_{\gamma}^2 \| x_s - x_s^* \|^2.
\end{eqnarray}
Combining (\ref{theo5:5}) and (\ref{theo5:6}) yields
\begin{eqnarray}\label{theo5:7}
  \frac{\rho}{4 L_{\gamma}^2} \| \nabla f(x_s) \|^2 \leq \frac{\rho}{4} \| x_s - x_s^* \|^2 \leq f(x_s) - f(x_{s+1})  + \frac{\eta_1 \sigma^2}{s N}.
\end{eqnarray}
Define $w_s = s$ and $\Delta_s = f(x_s) - f(x_{s+1})$. Multiplying both sides by $w_s$, we have
\begin{eqnarray}\label{theo5:8}
  \frac{\rho w_s}{4 L_{\gamma}^2} \| \nabla f(x_s) \|^2  \leq w_s \Delta_s  + \frac{ w_s \eta_1 \sigma^2}{s N}.
\end{eqnarray}
After telescoping (\ref{theo5:7}) for $s = 1, 2, \cdots, S$, we get
\begin{eqnarray}\label{theo5:9}
  \sum_{s=1}^S w_s \| \nabla f(x_s) \|^2 \leq \frac{4 L_{\gamma}^2}{\rho} \left( \sum_{s=1}^S w_s \Delta_s  + \sum_{s=1}^S \frac{ w_s \eta_1 \sigma^2}{s N} \right).
\end{eqnarray}
Taking the expectation w.r.t $s \in \{1, 2, \cdots, S\}$ with probability $p_s = \frac{s}{1+2+\cdots+S}$, we have
\begin{eqnarray}\label{theo5:10}
  \mathbb{E} \| \nabla f(x_s) \|^2 \leq \frac{4 L_{\gamma}^2}{\rho} \left( \frac{\sum_{s=1}^S w_s \Delta_s}{\sum_{s=1}^S w_s}  + \frac{\sum_{s=1}^S \frac{ w_s \eta_1 \sigma^2}{s N}}{\sum_{s=1}^S w_s} \right).
\end{eqnarray}
Based on the definition of $w_s$ and $\Delta_s$, setting $w_0 = 0$, we have
\begin{eqnarray}\label{theo5:11}
  \sum_{s=1}^S w_s \Delta_s &=& \sum_{s=1}^S w_s \left( f(x_s) - f(x_{s+1}) \right) 
  =\sum_{s=1}^S f(x_s) - S f(x_{S+1})  
  \nonumber\\
  &\leq& S (f(\bar{x}) - f(x_{S+1})) \leq w_S (f(\bar{x}) - f(x^*)),
\end{eqnarray}
where $\bar{x} = arg \max_{x_i, i \in [S]} f(x_i)$. Substituting (\ref{theo5:11}) into (\ref{theo5:10}), we get
\begin{eqnarray}\label{theo5:12}
  \mathbb{E} \| \nabla f(x_s) \|^2 &\leq& \frac{4 L_{\gamma}^2}{\rho} \left( \frac{w_S (f(\bar{x}) - f(x^*))}{\sum_{s=1}^S w_s}  + \frac{\sum_{s=1}^S \frac{ w_s \eta_1 \sigma^2}{s N}}{\sum_{s=1}^S w_s} \right)
  \nonumber\\
  &=& \frac{8 L_{\gamma}^2}{\rho} \left( \frac{f(\bar{x}) - f(x^*)}{S+1}  + \frac{ \eta_1 \sigma^2}{(S+1) N} \right).
\end{eqnarray}
As $T_s = s T_1$, we have
\begin{eqnarray}\label{theo5:15}
  T = T_1 + T_2 + \cdots + T_S = T_1 (1 + 2 + \cdots + S) = T_1 \frac{S(S+1)}{2} \leq T_1 \frac{(S+ 1)^2}{2}.
\end{eqnarray}
Substituting $S+1 \geq \sqrt{\frac{2T}{T_1}}$ into (\ref{theo5:12}), we get
\begin{eqnarray}\label{theo5:16}
  \mathbb{E} \| \nabla f(x_s) \|^2 &\leq& \frac{8 L_{\gamma}^2}{\rho} \left( \frac{(f(\bar{x}) - f(x^*))}{\sqrt{\frac{2T}{T_1}} }  + \frac{ \eta_1 \sigma^2}{\sqrt{\frac{2T}{T_1}} N} \right)
  \nonumber\\
  &=& O \left( \frac{\left( f(\bar{x}) - f(x^*) \right) \sqrt{T_1}}{\sqrt{T}} + \frac{\sqrt{T_1} \eta_1 \sigma^2}{N \sqrt{T}} \right)
  \nonumber\\
  &=& O \left( \frac{f(\bar{x}) - f(x^*) }{\sqrt{T \eta_1}} + \frac{\sqrt{\eta_1}\sigma^2}{N \sqrt{T}} \right),
\end{eqnarray}
where the last equality holds since $\eta_1 T_1 = 3 / \rho$. We use $\eta_1^N$ to denote the learning rate when using $N$ clients. Setting $\eta_1^N = N \eta_1^1$ yields
\begin{eqnarray}\label{theo5:17}
  \mathbb{E} \|\nabla f(x_s) \|^2 \leq O \left( \frac{1}{\sqrt{N T}} \right),
\end{eqnarray}
which completes the proof.

\end{proof}

\end{document}